\documentclass[a4paper,UKenglish,cleveref, autoref, thm-restate]{lipics-v2021}

\pdfoutput=1 
\hideLIPIcs  

\usepackage{booktabs}
\usepackage{graphics}
\usepackage{mathabx}
\usepackage{amsmath,amssymb,amsfonts}
\usepackage{mathtools}
\usepackage{footmisc}

\usepackage{color}

\newcommand{\domain}[1]{\ensuremath{\mathrm{Dom}_#1}}
\newcommand{\codomain}[1]{\ensuremath{\mathrm{Co}_#1}}
\newcommand{\pre}[1]{\ensuremath{\mathrm{Pre}}_#1}
\newcommand{\eff}[1]{\ensuremath{\mathrm{Eff}}_#1}

\newcommand{\length}[1]{\ensuremath{\mathrm{length}(#1)}}

\newcommand{\element}[1]{\ensuremath{\mathrm{element}(#1)}}

\newcommand{\CP}[1]{\ensuremath{\mathrm{CP}_{#1}}}

\title{Lifted Sequential Planning with Lazy Constraint Generation Solvers} 

\titlerunning{LCG Plan} 

\author{Anubhav Singh}{School of Computing and Information Systems \and University of Melbourne, Australia \and \url{https://github.com/anubhav-cs} }{anubhavs@student.unimelb.edu.au}{https://orcid.org/0000-0003-2770-801X}{This  research  was  supported  by  use  of  the  Nectar  Research Cloud, a collaborative Australian research platformsupported  by  the  National  Collaborative  Research  Infras-tructure Strategy (NCRIS). Anubhav Singh is supported by Melbourne Research Scholarship established by the University of Melbourne.}

\author{Miquel Ramirez}{Electrical and Electronic Engineering \and University of Melbourne, Australia \and \url{https://findanexpert.unimelb.edu.au/profile/778610-miquel-ramirez-javega} }{miquel.ramirez@unimelb.edu.au}{https://orcid.org/0000-0002-1838-7982}{}

\author{Nir Lipovetzky}{School of Computing and Information Systems \and University of Melbourne, Australia \and \url{https://nirlipo.github.io/}}{nir.lipovetzky@unimelb.edu.au}{https://orcid.org/0000-0002-8667-3681}{}

\author{Peter J. Stuckey}{Faculty of Information Technology \and Monash University, Australia \and \url{https://research.monash.edu/en/persons/peter-stuckey} }{peter.stuckey@monash.edu}{https://orcid.org/0000-0003-2186-0459}{}

\authorrunning{A. Singh, N. Lipovetzky, M. Ramirez and P. Stuckey} 

\Copyright{Anubhav Singh, Miquel Ramirez, Nir Lipovetzky, and Peter J. Stuckey} 



\category{} 

\relatedversion{} 

\nolinenumbers 

\begin{document}

\maketitle

\begin{abstract}
    This paper studies the possibilities made open by the use of
    Lazy Clause Generation (LCG) based approaches to Constraint Programming
    (CP)  for tackling \emph{sequential classical planning}.
    We propose a novel CP model based on seminal ideas on so-called
    lifted causal encodings for planning as satisfiability, that does
    not require grounding, as choosing groundings for functions and action schemas becomes
    an integral part of the problem of designing valid plans. This encoding does not
    require encoding frame axioms, and does not explicitly represent states as decision
    variables for every plan step.
    We also present a propagator procedure that illustrates the possibilities of LCG to
    widen the kind of inference methods considered to be feasible in planning as (iterated)
    CSP solving. We test encodings and propagators over classic IPC and recently proposed benchmarks for
    lifted planning, and report that for planning problem instances requiring fewer
    plan steps our methods compare very well with the state-of-the-art in optimal
    sequential planning.
    \end{abstract}

\section{Introduction}
\label{section:introduction}
One of the most significant advances in  classical planning
was the realisation of Green's~\cite{green:gps} vision for theorem proving as a
framework for general problem solving, via the ground-breaking seminal work
of H. Kautz and B. Selman~\cite{kautz:sat_early,kautz:encoding}.
Putting together the insights of  Blum and Furst~\cite{blum:graphplan}, where planning becomes
the problem of analysing whether goal states are reachable in a suitably defined graph,
with space efficient solutions to the \emph{frame problem} formulated in the
Situation Calculus~\cite{mccarthy:frame,haas:frame}, Kautz \& Selman showed that
formulating planning problems in terms of satisfiability of Conjunctive Normal Form (CNF)
formulas was feasible and, at the time, highly scalable. Attention to planning
as satisfiability has been somewhat eclipsed since then with the development of planning algorithms
based on direct, but lazy, incremental heuristic search over transition
systems~\cite{bonet:heuristics,hoffmann:ff,helmert:fd,richter:lama,torralba:symba}. Yet
deep theoretical connections exist between planning as satisfiability and as heuristic
search~\cite{geffner:planning_graphs,rintanen:admissible} questioning~\cite{rintanen:heuristics,suda:pdr} perceptions
of either approach as being parallel or mutually exclusive.

Like R. Frost's traveller in the woods this paper finds its way back to a crossroads in the
development and study of approaches to planning as satisfiability. One way is that  of the
so-called \emph{GraphPlan} encodings (linear and parallel),
the other being that of so-called \emph{causal} encodings (ground and lifted) invented
by Kautz, McAllester and Selman (KMS)~\cite{kautz:encoding}. Unlike Frost's poem though, lifted
causal encodings are
clearly an approach seldom followed in the literature
in automated planning. Inspired by the formulation of \emph{nonlinear} or partially-ordered planning
of McAllester and Rosenblitt~\cite{mcallester:nonlin},
these encodings have polynomial size  over the lower bound on the number of plan steps in
feasible plans. Importantly too, they do away entirely with the need for explanatory frame axioms. Yet
these very desirable properties follow from the premise of not having to ground actions and
predicates first, as otherwise the unavoidable exponential blow outs obliterate any practical
differences with GraphPlan encodings.

In this paper we realize this potential by tapping into the power of Lazy Clause Generation
(LCG)~\cite{ohrimenko:lcg}, a ground-breaking technology that unifies propositional satisfiability (SAT) and
Constraint Programming (CP), and allows representing implicitly large tracts of
complex systems of constraints by suitably defined inference procedures, or
\emph{propagators}. These lazily generate new constraints to record violations
by assignments to decision variables, and propagate information following from the
consistency of assignments and constraints in order to tighten the domains of variables. This, in addition to
very sophisticated and performant modeling tools and solvers~\cite{software:ortools},
provide us the foundations to develop scalable planners that follow the path laid by KMS lifted causal encodings. We have found
these planners to
clearly outperform state-of-the-art optimal planning algorithms on
benchmarks designed to be \emph{hard to ground}~\cite{correa:lifted}, while standing
their ground on the IPC benchmarks.

\textbf{Paper Overview}. We start the paper with a precise formulation of the kind of planning problems of
interest to us and their solutions. We then introduce a formalism to
describe the structure of states and actions that is based on Functional
STRIPS~\cite{frances:fstrips}. We assume  that all atoms in precondition
and effects are equality atoms over suitably defined function symbols and
constant terms. These ground \emph{domain theories} are then lifted~\cite{mcallester:nonlin}.
With these preliminaries in place, we introduce
our encoding, for which we prove validity and provide a characterisation of
its complexity. After that, we explain how we leverage state-of-the-art CP solvers
to implement efficiently our encoding. We then present a method to do a \emph{concise} transformation of planning instances represented in PDDL to FSTRIPS.
We end with an evaluation of several planners
built on top of our encoding and propagators, along with a brief
analysis of related work and a discussion of the significance and potential of this research.

\section{Formulation of Planning Problems}
\label{section:formulation}

A problem planning instance (PPI) is given by a tuple $P = (S, A, \to, s_0, S_G)$
where $S$ and $A$ are finite sets of states and actions,
$\to$ $\subset$ $S \times A \times S$ is the \emph{transition relation}, where
$s \to_{a} s'$ indicates that $s'$ is reachable from $s$ via $a$,
$s_0 \in S$ is the \emph{initial state} and $S_G \subset S$ is the set of designated \emph{goal states}.
We say that a PPI \emph{admits} a trajectory
$\sigma=s_0,a_1,s_1, \ldots, s_{i-1},a_{i}, s_i$ iff $s_{i-1}$ $\to_{a_{i}}$ $s_i$
for every $i > 0$. In this paper the notion of
\emph{planning problems} is that of \emph{optimization problems} where we seek \emph{sequences} $\sigma$ $=$ $s_0,a_1,s_1,\ldots,s_{k-1},a_k,s_k$,
that minimize $\length{\sigma} := k$, and satisfies $s_k \in S_G$.
The set of $\sigma$ sequences that satisfy $s_k \in S_G$ is referred to as the set of \emph{valid plans}
$\Pi_P$. The optimal cost of $P$ is thus $c^{*} := \min_{\sigma \in \Pi_P} \length{\sigma}$.
When $\Pi_P=\emptyset$,
we say that $P$ is infeasible and optimal cost is $c^{*} = \infty$.
We observe that $\vert S \vert - 1$ is a trivial upper bound on $c^{*}$ when $\Pi_P \neq \emptyset$.
Non-trivial, feasibly computable upper bounds are known~\cite{abdulaziz:upper_bounds} but only for PPIs with special structure.

\subsection{Factored Planning Problems}
\label{section:factored_planning}
A long recognised and adopted strategy to deal with large $\vert S \vert$ is that of \emph{factoring}
states and actions in PPIs as a preliminary step to develop
algorithms that can deal with large scale problems~\cite{pearl:heuristics,balcazar:96}.
We now present an account of FSTRIPS, a formalism to define such factorings, where a \emph{domain theory}
expresses assumptions on the structure of states and actions~\cite[Section 2.1]{frances:thesis}.

A PPI $P$ in Functional \textsc{Strips} is defined over a many-sorted first order logic theory with
equality, which we denote as ${\cal L}(P) = (T, \Phi, \Pi)$. The constituents of such a theory
capture relevant properties of and relations between objects, and provide the basic building blocks
to factorize states $s \in S$. $T$ is a finite non-empty set of finite sets called
\emph{types}, or \emph{sorts}, with a possibly infinite set of variables $x_{1}^t$, $x_{2}^t$, $\ldots$ for
each type $t \in T$. The \emph{universe} is the set $U = \cup_{t \in T} t$. $\Phi$ is a set of
\emph{function symbols} $f \in \Phi$, each of which is said
to have a \emph{domain} $\domain{f} \subset t_1 \times \dots t_i \times \times t_{d_f}$ where
$t_i \in T$, and a range $\codomain{f}$ $\in$
$T$. $\Pi$ is a set of \emph{relation symbols}. In this paper, $\Pi$ contains the standard relations
of arithmetic e.g. ``$<$'', ``$\leq$'', ``$>$'', ``$\geq$ '' in addition to equality ``$=$''.
Some PPIs may specify as well \emph{domain specific} relations, along with their denotation,
in addition to the former standard ones. We denote the maximum \emph{function arity} of the domain,
$\max_{f \in\Phi} d_f $,  as $K_f$.
States $s \in S$ in a \textsc{Fstrips} problem $P$ are \emph{semantic structures} that set the interpretation of
formulas over ${\cal L}(P)$ with fixed universe $U$. Thus each $s$ contains the \emph{graph}~\cite{bourbaki:set_theory}
of each function $f \in \Phi$, providing the denotation for \emph{functional terms} $f(\bar{t})$, where
$\bar{t} \in \domain{f}$. We note that $S(P)$ is a finite set since types $t \in T$ are finite sets too.

The transition relation $\to$ is specified via suitably defined \emph{action schemas} $\alpha \in Act$.
$\alpha$ is a non-logical symbol such that $\alpha \not\in \Phi \cup \Pi$. Actions $\alpha(\bar{x})$
capture sets of transitions in $\to$ parametrized by a tuple of typed
variables $\bar{x} = (x_{1}^{t_1}, \ldots, x_{d_\alpha}^{t_{d_\alpha}})$. $\tau(\alpha)$ denotes the tuple of types of parameters of $\alpha$, $\tau(\alpha)=(t_1, \ldots, t_{d_\alpha})$.
For each action schema $\alpha(\bar{x})$ we are given
$\pre{\alpha}(\bar{x})$ a \emph{precondition} formula over ${\cal L}(P)$ and variables $\bar{x}$.
In this
paper we consider a fragment of the formulas considered by~\cite{frances:thesis}, defined by the following
grammar in Backus-Naur form
\begin{align}
	\label{preconditions}
	\pre{\alpha}(\bar{x}) \coloneqq \top \,\mid\, \bigwedge_{i=1}^{card(\pre{\alpha})} f_{i}(\bar{y}_i) = z_i
\end{align}
where $\top$ is the tautology, $card(\pre{\alpha})$ denotes the \emph{number of equality atoms in the formula} $\pre{\alpha}$, $\bar{y}_i \subset \bar{x}$, and $z_i \in U$. Additionally, we are given an
\emph{effect} formula $\eff{\alpha}$ of the form
\begin{align}
	\label{effects}
	\eff{\alpha}(\bar{x}) \coloneqq \bigwedge_{j=1}^{card({\eff{\alpha}})} f_{j}(\bar{y}_j) = z_j
\end{align}
where $card(\eff{\alpha})$ denotes the \emph{number of equality atoms in the formula} $\eff{\alpha}$, $\bar{y}_j \subset \bar{x}$ and $z_i \in U$.

We now explain the FSTRIPS representation for the \emph{visitall} problem domain from the IPCs~\cite{icaps:competitions} in which an agent must visit all the cells in an $n \times n$ square grid starting from the center of the grid. In \emph{Visitall}, we only have one \emph{type}, $T:=\{C\}$, where $C$ is a set of cells in the grid, the set of \emph{function symbols} is $\Phi:=\{at, visited\}$, and it has a \emph{domain specific} relation $\Pi:=\{connected\}$. It has one action schema \emph{move} which allows the agent to move between two \emph{connected} cells of the grid, thus, $Act=\{move\}$. The \emph{move} schema takes two parameters which we denote by $\bar{x}:=(x_1, x_2)$, where $x_1$ is the current position of the agent and $x_2$ is the target position. The precondition and effect formulas of \emph{move} are defined as follows

\begin{subequations}
	\label{move_schema}
	\begin{align}
		\pre{\emph{move}}(\bar{x}) &\coloneqq at()=x_1 \land connected(x_1,x_2)=\top \label{move:pre} \\
		\eff{\emph{move}}(\bar{x}) &\coloneqq  at()=x_2 \land visited(x_2)=\top  \label{move:eff}
	\end{align}
\end{subequations}

\noindent The goal condition of \emph{visitall} requires the agent to {visit} all cells in $C$
\begin{align}
	\label{move:goal}
	Goal_{\emph{move}} \coloneqq  \bigwedge_{c \in C} visited(c) = \top
\end{align}

A set of action schemas $Act$ is \emph{systematic}~\cite{mcallester:nonlin,robinson:resolution_lifting} for a
PPI $P$ if and only if, for every $(s, a, s') \in \to$ there is an action schema $\alpha(\bar{x})$ such
that there exists a vector $\bar{v} \in t_{1} \times \dots \times t_{d_\alpha}$ and the following conditions hold:
\begin{subequations}
	\label{systematicity}
	\begin{align}
		s &  \models \pre{\alpha}(\bar{x})/\bar{v}\label{syst:prec} \\
		s' & \models \eff{\alpha}(\bar{x})/\bar{v}\label{syst:eff} \\
		s' & \models g(\bar{v}) = w,\,\text{when}~ s \models g(\bar{v})=w \land \not \exists w’, \eff{\alpha}(\bar{x})/\bar{v} \models g(\bar{v})=w’\label{syst:frame}
	\end{align}
\end{subequations}
where $\varphi(\bar{x})/\bar{v}$ is the (ground) formula that results from replacing every occurrence of $x_i \in \bar{x}$ by that
of $v_i \in \bar{v}$. The satisfiability relation in~\eqref{syst:prec}--\eqref{syst:frame} is defined in
a standard way~\cite{frances:thesis}. In words, a set of action schemas $Act$ is systematic whenever these
capture every possible reachability (or accessibility) relation between states $s$ and $s'$. We note that one
action schema $\alpha(\bar{x})$ can satisfy the above for many $(s_1, a_1, s_{1}')$, $(s_2, a_2, s_{2}')$, ...,
each of these tuples providing the semantics of \emph{ground action} $\alpha(\bar{v})$. In this paper we further assume
that action schemas do not change the denotation of any symbol in $\Pi$ (see section~\ref{section:cp_implementation} for a discussion
of their treatment in our encoding).

We denote the maximal arity of action schemas in $Act$ as $K_{\alpha} = \max_{\alpha \in Act} d_\alpha$.
The maximal number of equality atoms in precondition formulas is written as
$K_{\mathrm{pre}} = \max_{\alpha \in Act} card(\pre{\alpha})$ (resp.
$K_{\mathrm{eff}} = \max_{\alpha \in Act} card(\eff{\alpha})$ for effects).


\section{Planning as Satisfiability}
\label{section:sat_planning}

The approach known as \emph{planning as satisfiability}~\cite{kautz:sat_early}  proceeds by considering a sequence
of instances for a related decision problem, that of \emph{plan existence}.
We state the latter simply as follows: given some PPI $P$ and parameter $N_\pi$, with actions
and states defined in terms of some domain theory, the task is to prove that a feasible trajectory $\sigma$
exists such that $\length{\sigma} = N_\pi$, or alternatively, certify
that no such $\sigma$ exists. The classic algorithm
for optimal planning in this framework thus considers the sequence of CSPs $T_{P, n_0}$, $T_{P, n_1}$, $\ldots$,
$T_{P, n_k}$, $\ldots$ each of these a reduction of the plan existence problem for
$P$ and $N_\pi = n_k$ into that of the \emph{satisfiability} of a CSP $T_{P, n_k}$ with
suitably defined decision variables and constraints.
The sequence of natural numbers $n_0, n_1, \ldots, n_k, \ldots$ is typically, but not
necessarily~\cite{streeter:sat_search},
defined as $n_0$ $=$ $0$, $n_1$ $=$ $1$, and so on. When defined in this manner, as soon as
$T_{P,n_k}$ is satisfiable, we have proven that $c^{*} = n_k$. Scalable
certification of infeasibility in this framework has been an open problem until recently~\cite{eriksson:certification},
yet still remains challenging.

\section{Lifted Causal CP Model}
\label{lifted_causal_cp_model}

\begin{table}[t!]
	\centering
	\begin{tabular}{c|l}
		Symbol & Description \\
		\hline\\
		$act_i$ & Action assigned to slot $i$, $i=1,\ldots, N_\pi$ \\
		$active_{ik}$ & Pin is active \\
		$arg_{ij}$ & Value of $j$-th argument of slot $i$, $j=1,\ldots, K_{\alpha}$ \\
		$in_{ik}$ & $k$-th input pin data of slot $i$ \\
		$out_{il}$ & $l$-th output pin data of slot $i$ \\
		$spt_{jk}$ & Output pin supporting $k$-th input pin of slot $j$
	\end{tabular}
	\caption{Quick reference table for the decision variables in the model.
	$N_\pi$ is the number of slots, $K_\alpha$ is the (constant)  maximal number of arguments in any
	action schema $\alpha \in Act$. }
	\label{fig:notation_variables}
\end{table}


We start by giving a high-level account of our proposed encoding of CSPs $T_{P,N_\pi}$
and explain the roles played by the decision variables in
Table~\ref{fig:notation_variables}.
Central to our encoding is the notion
of plan step or \emph{slot}, of which we have one for \emph{each} action in
a plan. In contrast with the causal encoding of KMS,
slots are \emph{totally ordered} thus greatly simplifying the definition of persistence
that we use to deal with the frame axioms.
To each slot \emph{exactly one}
 action schema $\alpha \in Act$ needs to be assigned (variables $act_i$), which in
turn restricts choices on the possible values for the \emph{arguments}
(variables $arg_{ij}$) of the schema $\alpha$ as per $\tau(\alpha)$.
The assignments to the variables $act_i$ and $arg_{ij}$
 determine the choices of \emph{input and output pins} for a slot.
A \emph{pin} is a vector of decision variables which we use to represent equality
atoms $f(\bar{y}) = z$. These variables choose the function symbol $f$, terms
(constants in $U$) $\bar{y}$ (a vector) and $z$.
Input pins of slot $i$ thus encode
the equality atoms in $\pre{\alpha}(\bar{x})/\bar{v}$ required to be true in state $s_{i-1}$, and output
pins the atoms in $\eff{\alpha}(\bar{x})/\bar{v}$ required
to be true in $s_i$, where $\alpha(\bar{v})$ is the ground action selected by the assigned schema and
(possibly partially) assigned values to arguments.
These dependencies between the variables of a slot  are depicted in Figure~\ref{fig:slot_constraint_graph}(a),
and
Figure~\ref{fig:slot_constraint_graph}(b) illustrates the \emph{active} constraints between variables when the \emph{move} action schema is chosen at slot $i$ in the \emph{visitall} problem domain. The \emph{move} schema has two arguments representing the current position and the target position of an agent in an $n \times n$ grid. Thus, when $act_i=\emph{move}$, we require that $arg_{i1} \in C$ and $arg_{i2} \in C$, the input pin $in_{i1}$ is assigned the function symbol $at$ and the variable $y_{i1}$ holds an equality relation with $arg_{i1}$.

We note that
we create up front variables for arguments and pins following from
$K_{\alpha}$, $K_f$, $K_{\mathrm{pre}}$ and $K_{\mathrm{eff}}$, all
constants given by the data in an instance $P$. For a given slot $i$,
argument variables $arg_{ij}$ that are not used by the schema assigned are set
to a special $\mathrm{null}$ constant. To disable pins not needed to represent
atoms in preconditions or effect formulas we have a Boolean variable $active_{ik}$ that indicates
if they are being used. Finally, variables $spt_{jk}$
allow choosing what output pin is supporting a given input pin. These variables
allow encoding \emph{causal links}~\cite{mcallester:nonlin} without referring
explicitly to atoms.

\subsection{Variables and Constraints}
\label{section vars_and_constraints}

\begin{figure}[t!]
	\centering
	\hspace{-1cm}


{\includegraphics[width=1\textwidth]{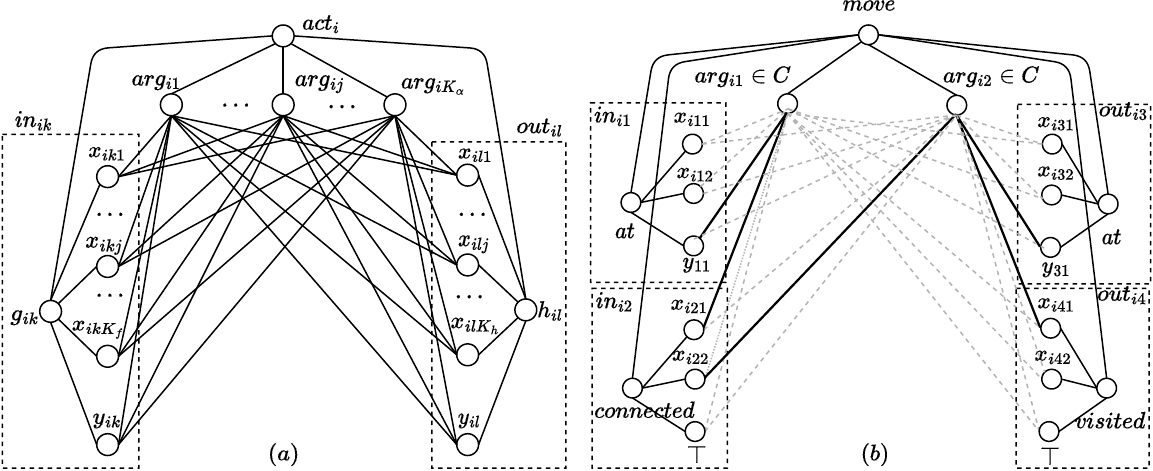} }

	\caption{$(a)$~Constraint graph depicting dependencies between the decision variables to model
		slots (plan steps). There is one vertex for every decision variable, and there is an edge
		between two variables whenever they appear together in the definition of at least one
		constraint. $(b)$~Active constraints when the \emph{move} action schema is chosen in \emph{visitall}.}
	\label{fig:slot_constraint_graph}
\end{figure}

We now give a formal and precise account of the variables and constraints in the model.
Let $N_\pi$ be the maximal number of slots in a valid plan. For every slot $i=1,\ldots,N_\pi$
we have integer variables $act_i \in [Act]$\footnote{We use the notation $[S]$ to designate the \emph{indexing} of
the elements of a set, e.g. $[S]=1,2,\ldots$ for $S=\{e_1, e_2, \ldots \}$.}
to choose the action schema $\alpha \in Act$ assigned to it. Argument variables $arg_{ij}$
 select the values of the variables introduced by lifting, and their domains
correspond to the types $\tau_j(\alpha)$, whenever $act_i = \alpha$
\begin{align}
	\label{argument_binding}
	(act_i = \alpha) \rightarrow arg_{ij} \in [\tau_j(\alpha)],\, j=1, \ldots, d_{\alpha}
\end{align}
The choice of action schema assigned to a slot restricts the choices of (ground) precondition and
effect formulas.
As introduced above, the input and output pins of a slot $i$ are
the vectors of decision variables that select function symbols, arguments and values
$in_{ik}$ $:=$ $(g_{ik}, x_{ik1}, \ldots, x_{ik\bar{N}_f}, y_{ik})$
and $out_{il}$ $:=$ $(h_{il}, x_{il1}, \ldots, x_{il\bar{N}_{f}}, y_{il})$.
Consistency of schema, preconditions and effects assigned to slot $i$ is enforced by
\begin{subequations}
	\label{input_and_output_pins}
	\begin{align}
		\label{pins:g_ik}
		(act_i = \alpha) &\rightarrow (g_{ik} = f_k) \land \mathrm{bind}(f_k, \bar{x}_{ik}, y_{ik}, \overline{arg}_i) \\
		\label{pins:h_il}
		(act_i = \alpha) &\rightarrow (h_{il} = f_l) \land \mathrm{bind}(f_l, \bar{x}_{il}, y_{il}, \overline{arg}_i)
	\end{align}
\end{subequations}
for $k=1, \ldots,  card(\pre{\alpha})$ and $l=1, \ldots, card(\eff{\alpha})$. $\overline{arg}_i$ is
the vector of argument variables for slot $i$. The predicate
$\mathrm{bind}$ in~\eqref{pins:g_ik} and~\eqref{pins:h_il} ensures that subterms $\bar{x}$ and $y$ of
equality atoms in preconditions and effects are consistent with the definition of schema $\alpha$, expanding into the following
\begin{align}
	\label{pins:bind_constraint}
	\bigwedge_{j=1}^{d_{f}} \bigg( \bigvee_{j_2=1}^{K_{\alpha}} (x_{j} = {arg}_{ij_2}) \bigg)
	\land \bigvee_{j_2=1}^{K_{\alpha}} (y = {arg}_{ij_2})
\end{align}
\noindent The dependencies induced by constraints~\eqref{argument_binding}--\eqref{pins:bind_constraint} are
depicted in Figure~\ref{fig:slot_constraint_graph},

States are represented implicitly in our model, and in order to ensure that no change on the
initial state $s_0$ is allowed without an event in the plan that explains any changes we
rely on the notion of causal consistency or \emph{persistence}
\begin{definition}
	\label{def:persistence}
	Let $f(\bar{x})$ be a functional term $f \in \Phi$ and $y$ a valid value in $\codomain{f}$. We say that
	the \emph{truth} of atom $f(\bar{x}) = y$ \emph{persists} between slots $i$ and $j$, $0 \leq i < j \leq N_\pi$,
	if (1) $f : \bar{x} \mapsto_{s_i} y$, and (2) for every slot $j'$, $i < j' <j$, there is
	no pin $out_{j'l} = (h_{j'l}, \bar{x}_{j'l}, y_{j'l})$ such that $h_{j'l} = f$, $\bar{x}_{j'l} = \bar{x}$, and $y_{j'l} \neq y$.
\end{definition}
This ensures that states $s_j$ are given either by 1) function and value assignments in the initial state $s_0$ that have \emph{persisted}
through (ground) actions $a_i$ encoded in slots $i$, $0$ $<$ $i$ $<$ $j$, or 2) an assignment made by some action $a_i$ encoded in
some slot $i$, $0$ $<$ $i$ $<$ $j$, that has persisted through the actions $a_{j'}$ encoded by slots $j'$, $i$ $<$ $j'$ $<$ $j$. For example, in \emph{visitall}, an atom $at()=c_1$ in the intial state, where $c_1$ is the initial position of the agent, is said to have \emph{persisted} until slot $3$ iff the atom \emph{holds} at slot $1$ and $2$.

The many
possible cause-and-effect relations between output and input pins that may justify the causal consistency of a plan are modelled
via so-called (causal) \emph{support} variables, $spt_{jk}$ $\in$ $\big( [0, j-1] \times [1, K_{\mathrm{eff}}] \big) \cup \{ \mathrm{null} \}$, for each slot
$0$ $<$ $j$ $\leq$ $N_\pi$ and input pin $1$ $\leq$ $k$ $\leq$ $K_{\mathrm{pre}}$. The domain of these variables is the set of two-dimensional vectors whose
first element is the index of the slot $i$ and the second the index $l$ of the output pin supporting $in_{jk}$, plus a dummy vector $\mathrm{null}$
that indicates that input pin $in_{jk}$ does not have a causal support assigned.
The following constraints enforce that every active pin has a matching supporting one
\begin{subequations}
	\label{persistence}
	\begin{align}
		spt_{jk} = \mathrm{null} &\rightarrow \, \neg active_{jk} \label{persist:1} \\
		spt_{jk} = (i,l) &\rightarrow \, out_{il}=in_{jk} \label{persist:2} \\
		spt_{jk} = (i,l) &\rightarrow \, \bigwedge_{i < j' < j} \mathrm{persists}(out_{j'l}, in_{jk}) \label{persist:3}
	\end{align}
\end{subequations}
\noindent
for each $1$ $\leq$ $j$ $\leq$ $N_{\pi}$ and $0$ $\leq$ $i$ $<$ $j$. Constraint~\eqref{persist:1} excuses input
pins that are inactive from having a causally supporting output pin. Constraint~\eqref{persist:2}
ensures that the values for functions, domain and valuation set by pins $in_{jk}$ and $out_{il}$
are matching. Constraint~\eqref{persist:3} encodes the requirement of causal supports to not
be interfered by any ground action set for intermediate slots $i < j' < j$.
$\mathrm{persists}(out_{il}, in_{jk})$ in constraint~\eqref{persist:3} expands into the formula
\begin{align}
	\label{eq:persistence}
	(h_{il} \neq g_{jk}) \lor (\bar{x}_{il} \neq \bar{x}_{jk}) \lor (y_{il} = y_{jk})
\end{align}
\noindent In the specific case of \emph{visitall}, the constraints~\eqref{persistence} and \eqref{eq:persistence} impose additional constraints on the input and output pin variables than the ones depicted in Figure~\ref{fig:slot_constraint_graph}(b), thus, restricting the choices of $act_i$ and $\overline{arg}_i$ as well. For example, if $spt_{31} = (1,3)$, then the constraint \eqref{persist:2} requires that the assignment to $out_{13}:=(at, y_{13})$ matches that of $in_{31}:=(at, y_{31})$, i.e. $y_{13}=y_{31}$, and the constraints~\eqref{persist:3} and \eqref{eq:persistence} ensure that the assignment to $out_{13}$ \emph{persists} through slot $2$, i.e. the (ground) action $a_2$ does not affect the interpretation of $at$. It is easy to see that when $spt_{31} = (1,3)$, the choice of $y_{13}$ in turn affects $arg_{31}$ since $arg_{31}$ holds an equality relation with $y_{31}$.

\noindent Additionally, whenever we assign action schemas $\alpha$ to slots $i$ we mark input
and output pins as being active
\begin{subequations}
	\label{pin_activity}
	\begin{align}
		&act_i = \alpha \rightarrow active_{ik},\, act_i = \alpha \rightarrow \neg active_{ik'} \\
		&act_i = \alpha \rightarrow active_{il},\, act_i = \alpha \rightarrow \neg active_{il'}
	\end{align}
\end{subequations}
with indices ranging as follows: $1$ $\leq$ $k$ $\leq$ $card(\pre{\alpha})$, $card(\pre{\alpha})$ $<$ $k'$ $\leq$ $K_{\mathrm{pre}}$,
$1$ $\leq$ $l$ $\leq$ $card(\eff{\alpha})$, $card(\eff{\alpha})$ $<$ $l'$ $\leq$ $K_{\mathrm{eff}}$.
Initial and goal states are accounted for in the following way. To model the initial state, we define
slot $0$ to consist exclusively of a set of (output) pins $out_{0l}$ where $l$ ranges over the indexing
of the set
\begin{align}
	\label{init_state_extension}
	{\cal F} := \bigcup_{f \in \Phi} \domain{f}
\end{align}
\noindent and each pin is set as per constraints, for each $l \in [{\cal F}]$
\begin{align}
	\label{init_assignemnts}
	out_{0l} = (f_l, \bar{x}_l, y)
\end{align}
\noindent and $y \in \codomain{f}$ such that  $f_l : \bar{x}_l \mapsto_{s_0} y$.
Goal states are modelled too by having the slot $N_\pi$ to have special structure,
in this case by only having input pins $in_{N_\pi k}$ with $k$ ranging over the indexing
of the set of equality atoms $\varphi_k \equiv f_k(\bar{x}_k) = y_k$ in the goal formula
\begin{align}
	\label{goal_constraints}
	\bigwedge_{\varphi_k \in Goal} in_{N_\pi k} = (f_k, \bar{x}_k, y_k) \land active_{N_\pi k}
\end{align}

\subsection{Analysis: Systematicity and Complexity}
\label{section:analysis}

We first establish as fact that our encoding is \emph{systematic}
~\cite{mcallester:nonlin}.

\begin{theorem}[Systematicity]
	Let $T_{P,N_{\pi}}$ $=$ $({\cal X}$, ${\cal C})$ be the CSP given by the decision variables ${\cal X}$ in Table~\ref{fig:notation_variables}
	and set of constraints ${\cal C}$~\eqref{argument_binding}--\eqref{goal_constraints}, for some suitable choice of $N_{\pi}$.
	There exists an assignment $\xi$ onto variables ${\cal X}$ that satisfies every constraint in ${\cal C}$ if an only if there exists
	a feasible solution $\sigma$ to the PPI $P$, whose actions are given by slots, argument, input and output pin variables values.
\end{theorem}

\begin{proof}
	See \ref{theorem:systematicity}.
\end{proof}

Giving bounds on the number of variables generated for a CP model is not as informative as doing so for CNF formulas due
to the advanced \emph{pre-solving techniques} that state-of-the-art CP solvers employ~\cite{software:ortools}, that may
introduce auxiliary Boolean variables and eliminate variables whose values can be determined without searching. Assuming that
no such transformations are applied to the CSP, the set of $spt_{jk}$ variables and their domain constraints, which
are required to be encoded explicitly by LCG solvers~\cite{ohrimenko:lcg}, is the largest and is $O(N_{\pi}^2 K_{max})$ where
$K_{max}$ $=$ $\max$ $\{$ $K_{\mathrm{pre}}$, $K_{\mathrm{eff}}$ $\}$. In the IPC benchmarks, this results often in just
a quadratic rate of growth as $K_{max}$ is usually much smaller than $N_\pi$ for optimal plans. Yet, as we will see in
our Evaluation, this is not always the case, and we get cubic rates.

\section{Programming the Model with \textsc{CpSat}}
\label{section:cp_implementation}
To implement the CP model introduced in the previous section we have used the LCG solver
bundled in Google's \textsc{Or-Tools} package, \textsc{CpSat}~\cite{software:ortools}. At the time
of writing this, \textsc{CpSat} is the state-of-the-art LCG solver as adjudicated by the
latest results of the \textsc{Minizinc} challenge~\cite{stuckey:challenge}. A key feature
of \textsc{CpSat} we rely on is the extensive support for different variants of
so-called \element{\cdot} constraints~\cite{book:hooker:2012}.
These constraints implement \emph{variable indexing}, a key modeling feature in CP, and
correspond with the statement $\bar{v}_{x} = z$, that reads as
``the element at position $x$ of vector $\bar{v}$ must be equal to $z$''. The power
of these constraints lies in the possibility of elements of $\bar{v}=(v_1,\ldots,v_m)$, $x$ and
$z$ being \emph{all} decision variables.
We write \element{\cdot} constraints using Hooker's~\cite{book:hooker:2012} notation
\begin{subequations}
\label{element:definition}
\begin{align}
	&\element{x, z \mid \bar{v}} \label{element:1Dform}\\
	&\element{(x_1, x_2), z \mid V} \label{element:2Dform}
\end{align}
\end{subequations}
\noindent where \eqref{element:1Dform} implements the statement $\bar{v}_x = z$, and
\eqref{element:2Dform} implements $V(x_1,x_2) = z$ that reads as ``the element of matrix $V$
at coordinates $x_1,x_2$ must be equal to $z$''. $x$, $x_1$ and $x_2$ are thus \emph{index}
variables that locate one decision variable within a collection.

In our implementation, the action assigned to slot $i$, $act_i$ is an index variable, which
depends on the data assigned to input and output pins as per constraints ~\eqref{input_and_output_pins} and \eqref{pins:bind_constraint}.
Both of these constraints can be encoded compactly using \eqref{element:1Dform}
\begin{subequations}
	\label{element:action_binding}
	\begin{align}
		&\mathrm{element}(act_i, g_{ik} \mid (f_1, f_2,\ldots, f_j, \ldots,f_{|Act|})) \label{element:action_input_pin_function} \\
		&\mathrm{element}(act_i, x_{ikl} \mid (\nu_1, \nu_2,\ldots, \nu_j\ldots,\nu_{|Act|})) \label{element:action_input_pin_args} \\
		&\mathrm{element}(act_i, y_{ik} \mid (\mu_1, \mu_2,\ldots,\mu_j \ldots,\mu_{|Act|})) \label{element:action_input_pin_value}
	\end{align}
\end{subequations}

where $f_j$ is a function symbol, $\nu_j$ and $\mu_j$ are the argument variables of slot $i$ or constants that are consistent with the definition of action schema when $act_i=j$. $f_j$ is assigned a special function when the input pin is \emph{inactive} in the action schema, thus accounting for literals $\neg \textit{active}_{ik}$.

\noindent To give the readers a better intuition of the element constraints above, we revisit the example of \emph{visitall}. Visitall only has one action schema \emph{move}, and hence, the element constraints are uncomplicated. Consider the input pin $in_{i1}$ in the Figure~\eqref{fig:slot_constraint_graph}(b), the constraint \eqref{element:action_input_pin_function} for the pin is written as $\mathrm{element}(act_i, g_{i1} \mid (at))$, the constraint \ref{element:action_input_pin_args} as $\mathrm{element}(act_i, x_{i1l} \mid (\Box))$, and lastly the constraint \eqref{element:action_input_pin_value} as $\mathrm{element}(act_i, y_{i1} \mid (arg_{i1}))$, where $\Box$ is a \emph{dummy} symbol which accounts
for \emph{inactive} variables
These constraints then ensure that the variables of the slot $i$ are internally consistent for all possible assignments to $act_i$.


The \emph{support} variables $spt_{jk}$ are index variables with two elements, $(i, l)$, the first being the index of a slot
and the second the index of an output pin. As discussed in the previous section, variables $spt_{jk}$ depend on the data of
input pin $in_{jk}$. Output pins $out_{jl}$ and constraints~\eqref{persist:1} and~\eqref{persist:2} can be accounted for
using \element{\cdot} in its matrix form~\eqref{element:2Dform}
\begin{align}
	\label{element:support_binding}
	\mathrm{element}((i, l), in_{jk} \mid H)
\end{align}
where $H$ is the data of all output pins in a $2$-dimensional grid. Constraint \eqref{element:support_binding}
thus requires that the data of the output pin at the row $i$ (slot $i$) and column $l$ ($l$-th pin) of $H$ is equal to
that in $in_{jk}$. The matrix $H$ includes a specially defined element, containing the data used to represent \emph{inactive} pins, whose index is assigned to $spt_{jk}$ when the $k$-th input pin at slot $j$ is \emph{inactive}, thereby encoding constraint \eqref{persist:1}.

We exploit other modeling features offered by \textsc{CpSat} to implement the $\mathrm{persist}(\cdot)$ predicate
given in Equation~\eqref{eq:persistence}
\begin{subequations}
	\label{persists_formula_implementation}
	\begin{align}
		u \lor  v_1 &\lor \ldots \lor v_o \lor \ldots \lor v_{d_f} \lor w \label{boolor:persists_formula} \\
		u \rightarrow &(h_{il} \neq g_{jk}) \label{notequal:input_output_function} \\
		{v_o}  \rightarrow &(x_{ilo} \neq x_{jko}) \label{notequal:input_output_arguments} \\
		w \rightarrow &(y_{il} = y_{jk}) \label{equality:input_output_value}
		\end{align}
\end{subequations}
where $u, v_o$ and $w$ are auxiliary Boolean variables created for each $(out_{j'l}, in_{jk})$ pair, and
$1$ $\leq$ $o$ $\leq$ $K_f$.
\textsc{CpSat} does not represent explicitly constraints \eqref{notequal:input_output_function},
\eqref{notequal:input_output_arguments} and \eqref{equality:input_output_value}. Instead, it
collects them into a \emph{precedences propagator} as \emph{inequalities} between integer variables.
The \emph{precedences propagator} uses the Bellman-Ford algorithm to detect negative cycles in the
constraint graph of inequalities, and propagates bounds on the integer variables~\cite{nieuwenhuis:diff_logic}.
Still, $O(N_{\pi}^2 K_{\mathrm{pre}} K_{\mathrm{eff}} K_f)$ variables are
generated in the worst-case, e.g. $O(n^5)$ if all these quantities belong to the same order of magnitude.
In most of the benchmarks we use to test planners using these encodings, the number of preconditions,
effects and arity of functions are much smaller than $N_\pi$, thus generating $O(n^2)$ variables.

\textbf{Encoding \emph{static} relations with \emph{table constraints}}. Some PPIs contain \emph{domain specific} relations that are not affected by the action schemas. For example, in \emph{visitall} the relation \emph{connected} is \emph{static} and its interpretation is fixed by the specification of the initial state.
We encode the dependency of the input pin variables on these static relations using \emph{table constraints} which allow us to specify a set of \emph{allowed} (or \emph{forbidden}) assignments to a tuple of variables, i.e. for a static relation $R$, we encode the constraint as $table(in_{ik}, R)$, where the pin $in_{ik}$ is specially designed to represent a tuple in $R$. This is more efficient than using the element constraints~\eqref{element:support_binding} since \textsc{CpSat} translates them into a concise CNF formulation.

\textbf{Propagator for Required Persistence}. We recall constraint~\eqref{persist:3} that enforces
the requirement that whenever an output pin $out_{il}$ is to provide causal explanation for an input pin $in_{jk}$, the atom described by the
former is not interfered with by any other output pins in intermediate slots. Clearly, the number
of variables and clauses~\eqref{boolor:persists_formula} is proportional to $j - i$, bringing the potential
number of variables generated to be $O(n^6)$. While such a rate of growth seems unsustainable for non-trivial
instances, the empirical results we obtain clearly show that this worst-case does not always follow for many
domains widely used to evaluate planning algorithms.

To avoid this blow up, yet keep the strong inference offered by the system of
constraints~\eqref{persists_formula_implementation}, we introduce specific propagator that
interfaces with \textsc{CpSat} \emph{precedences} propagator and checks whether
constraint~\eqref{persist:3} is satisfied. The propagator activates whenever an assignment in the CDCL
search fully decides the variables in input pin $in_{jk}$. We then check, for every slot $j'$, $0$ $<$ $j'$ $<$ $j$,
whether the current assignment has fully decided some output pin $out_{j'l}$, and proceed to evaluate
the $\mathrm{persistence}$ predicate in Equation~\eqref{persistence} on the assignment. If the former
evaluates to false, we have a conflict between the assignment and constraint~\eqref{persist:3}.
We then generate the following blocking clause or \emph{reason} to explain it
\begin{align}
	\label{required_persistence_reason}
	\neg(\varphi_{out_{j'l}} \land \varphi_{in_{jk}}) \lor \big( LB(spt_{jk_1}) > j' )
\end{align}
where $\varphi$ formulae are the conjunction of equality atoms that bind variables in pins to the values
in the current assignment. $spt_{jk_1}$ is the first element (variable) in $spt_{jk}$ and $LB$ is
a function provided by \textsc{CpSat} that allows to access the lower bound of the domain of a variable
in constant time.

\textbf{Searching for plans}. Our algorithm for planning as satisfiability uses a
strategy to find plans that is analogous to the notion of \emph{lookaheads} in Approximate
Dynamic Programming. Like in the classic algorithm described earlier in the paper,
we generate a sequence of CSPs $T_{P,k_1}$, $T_{P,k_2}$, $\ldots$, $T_{P,k_i}$, $\ldots$
with $k_0=0$ and $k_i - k_{i-1} \geq 1$ for $i > 0$.

To each $T_{P,k_i}$ we impose the
Pseudo-Boolean objective
\begin{align}
	\label{pbo_objective}
	\sum_{j=1}^{k_i} enabled_j
\end{align}
where $enabled_j$ are auxiliary Boolean variables which we use to ``switch off'' slots
via constraints $\neg enabled_j$ $\rightarrow$ $act_j > \vert Act \vert$. If \textsc{CpSat}
proves that the  resulting \emph{optimization} problem has \emph{finite} optimal value
$z$ then we know that $c^{*} = z$ and the search terminates as we have found an optimal
plan.
If \textsc{CpSat} finds an upper bound $z$, that is, a sequence of feasible solutions
are found but it is not possible to prove optimality of the last one within the time limit set,
then we know that $c^{*} \leq z$. If \textsc{CpSat} proves the tightest upper bound
to be $\infty$ (e.g. the problem is unsatisfiable) then we have proved a
\emph{deductive lower bound}~\cite{geffner:planning_graphs}, as we know that $c^{*} > k_i$.
In this last case, we repeat the process above with $TP_{k_{i+1}}$ until a solution is
found, or the allowed time to search for plans is exhausted.

\section{Functional transformation of a PDDL task}
\label{Functional transformation}

Benchmarks in planning are not expressed in FSTRIPS directly but in PDDL~\cite{haslum2019introduction}, a defacto abstract representation of a PPI used by the research community. PDDL is defined by  the tuple $P = \left< {U, T, {\cal P}, Act, s_0, \gamma}  \right>$, where $U$ is a set of \emph{objects}, $T \in 2^{U}$ is a collection of types, ${\cal P}$ is a set of \emph{predicates}, $Act$ is a set of \emph{action schemas}, $s_0$ is the initial state, and $\gamma$ is \emph{goal formula}. For a given predicate $\mathsf{P} \in {\cal P}$ of arity $d_\mathsf{P}$, there are two associated literals, the positive atom $\mathsf{P}(\bar{x})$ and negative atom $\neg \mathsf{P}(\bar{x})$, where $\bar{x}$ is tuple of variables $(x_1, x_2,\ldots,x_{d_\mathsf{P}})$, the domain of $x_i$ corresponds to a type $t_i \in T$. We denote $\max_{\mathsf{P}\in {\cal P}}d_\mathsf{P}$ as $K_\mathsf{P}$. A \emph{simple} reformulation of PDDL planning task into FSTRIPS representation is to treat each predicate $\mathsf{P} \in {\cal P}$ as a Boolean function or a \emph{mapping}, $f_\mathsf{P}: \domain{{f_\mathsf{P}}} \mapsto \codomain{{f_\mathsf{P}}}$, $\domain{{f_\mathsf{P}}} \subseteq t_1 \times t_2 \times,\ldots, \times t_{d_\mathsf{P}}$, $\codomain{{f_\mathsf{P}}} = \{\top, \bot\}$. Thus, $(f_\mathsf{P}(\bar{x})=\top)$ denotes $\mathsf{P}(\bar{x})$ and $(f_\mathsf{P}(\bar{x}) = \bot)$ denotes $\neg \mathsf{P}(\bar{x})$. Mappings that go beyond Boolean functions can yield a more compact FSTRIPS representation for CP. Hence, we present in this section a novel method to derive a more concise functional transformation of predicates ${\cal P}$.

Any function $f : \domain{f} \mapsto \codomain{f}$ has an associated binary relation, a \emph{mapping}, ${\cal R}_f:= \{(x, y) \mid f(x)=y,\ x \in \domain{x}, y \in \domain{y} \}$. A predicate $\mathsf{P} \in {\cal P}$  of arity $2$ also has an associated binary relation, ${\cal R}_\mathsf{P} := \{(x, y) \mid \mathsf{P}(x, y),\ x \in \domain{x}, y \in \domain{y}\}$. If the relation ${\cal R}_\mathsf{P}$ is a mapping, then we can create a \emph{more concise} transformation, i.e. $f_\mathsf{P}: \domain{x} \mapsto \domain{y}$, instead of $f_\mathsf{P}: \domain{x} \times \domain{y} \mapsto \{ \top, \bot\}$. Moreover, $0$-ary, $1$-ary and $n$-ary predicates can all be mapped into the binary case without loss of generality, i.e. $\mathsf{P}()$ can be substituted by $\mathsf{P}(c_1, c_2)$, $c_1, c_2 \in \domain{c}, \domain{c} \cap U = \emptyset$, $\mathsf{P}(y)$ by $\mathsf{P}(c, y)$, and for $n\geq2$, $\mathsf{P}(x, y)$ replaces the predicate $\mathsf{P}(u_1,\ldots, u_n)$, where $x$ is a tuple in the set of combinations of parameters of length $n-1$ and $y$ is the parameter which is excluded from $x$.
For example, in the PDDL specification of \emph{visitall}, $at$ is a  $1$-ary predicate which represents the current position of the agent. We can map $at$ into the binary case by introducing a constant $A$, and then, transform it into a function as $at:\{A\} \mapsto C$ since the agent can take at most one position on the grid.
Thus, for each predicate $\mathsf{P} \in {\cal P}$, there is an associated binary relation ${\cal R}_\mathsf{P} := \{(x, y) \mid \mathsf{P}(x, y), x \in \domain{x}, y \in \domain{y}\}$. There are two \emph{necessary and sufficient} conditions for a binary relation to be a \emph{mapping}, $(1)$ it is \emph{right-unique}, and $(2)$ it is \emph{left-total}. A binary relation ${\cal R}_\mathsf{P}$ is \emph{right unique} iff $\forall\ x_1, x_2 \in \domain{x}, y_1,y_2 \in \domain{y},\ \mathsf{P}(x_1, y_1) \land \mathsf{P}(x_2, y_2) \land (x_1=x_2) \rightarrow (y_1=y_2)$. It is \emph{left-total} iff $\forall\ x \in \domain{x},\ \exists\ y \in \domain{y},\ \mathsf{P}(x, y)$. Since, the states $s \in S$ set the interpretation of predicate $\mathsf{P}$, we have to at-least prove that the \emph{right-unique} and \emph{left-total} conditions \emph{hold} in $S^R_P$, the set of states reachable from $s_0$, to make a case for the \emph{more concise} functional transformation.

\begin{theorem}
	If ${\cal R}_\mathsf{P}$ is a \textit{mapping} in all possible interpretations $s \in S^R_P$, then ${P'}$, the transformation of the problem $P$ which encodes the predicate $\mathsf{P}$ as a function $f_\mathsf{P} : \domain{x} \mapsto \domain{y}$, has the same set of reachable states as $P$, i.e. $S^R_P = S^R_{P'}$
\end{theorem}

\begin{proof}
	See \ref{theorem:function_transformation}
\end{proof}

A \emph{sufficient} condition for the \emph{right-unique} property to hold in all interpretations $s \in S_{P}^R$ is that the \emph{negation} of \emph{right-unique} condition is \emph{false} in $S^R_{P_r} \supseteq S_{P}^R$, where $P_r$ is a \emph{relaxation} of $P$. Thus, we can do a \emph{relaxed} reachability analysis of the formula  $\psi_\mathsf{P} := \exists x_1, x_2 , y_1,y_2,\ \mathsf{P}(x_1, y_1) \land \mathsf{P}(x_2, y_2) \land (x_1=x_2) \land (y_1 \neq y_2)$ to test whether the \emph{right-unique} condition holds for all $s \in S_{P_r}^R$, i.e. the \emph{right-unique} property \emph{holds} if $\psi_\mathsf{P}$ is unreachable in $P_r$.
The \emph{relaxation} is important since checking the reachability of the condition in $P$ is as hard as solving the problem itself.
To this end, we extend the $h^m$ heuristic~\cite{geffner2000admissible}, which is an \emph{admissible approximation} of the
optimal heuristic function $h^*$, to the first order logic existential
formula of the form $\psi := \exists \bar{x}, \psi^L \land \psi^{EQ}$, where
$\bar{x}:=(x_1, x_2,\ldots, x_n)$ is a vector of parameters, $\psi^L$ is a
conjunction of literals whose interpretation is set by the states $s \in
S^R_{\Pi}$, and $\psi^{EQ}$ is an equality-logic formula\footnote{See the
  definition of {$h^m$} over first-order logic formulae in the Appendix \ref{h_m:fol_existential}.}. We then use the extension of $h^m$ to test the reachability of the formula $\psi_\mathsf{P}$, i.e. if $h^m(\psi_\mathsf{P})=\infty$, then $\psi_\mathsf{P}$ is unreachable, and hence, the right unique condition holds in all interpretations $s \in S_P^R$ of $\mathsf{P}$.  Also, we note that, if ${\cal R}_\mathsf{P}$ is \textit{right-unique}, the \textit{left-total} property is trivial to satisfy. For each $x \in \domain{x}$, if $\not \exists y \in \domain{y},\ \mathsf{P}(x, y)$, then we map $x$ to $\Box$, a \textit{dummy} constant symbol.

$h^m(\psi_\mathsf{P})$ can be efficiently computed using dynamic programming with memoization. For a fixed value of $m$, the complexity of the above procedure is low polynomial in the number of nodes, i.e. the number of first-order logic formulas $O({|{\cal P}|^m}\cdot 2^{K_{\cal P}})$. An important property of the procedure to compute $h^m$ is that it is \emph{entirely lifted}, i.e. no \emph{ground} atom would occur in the formulas obtained through regression if no action schema has \emph{ground} atoms. This is specially useful in \emph{hard-to-ground}(HTG) domains where the size of \emph{ground} theory renders the \emph{translation} methods \cite{helmert:fd} used by most planners intractable.

\section{Evaluation}
\label{evaluation}
Our experiments consist in running a given planner on a PPI, ensuring that the solver process runs on a
single CPU core (Intel Xeon running at 2GHz). We impose resource usage limits both on runtime ($1800$ s) and memory
($8$ GBytes).  We used~\cite{seipp:lab} \emph{Downward Lab's} module to manage the parallel
execution of the experiments.


We compare the performance of \textsc{CpSat} solving our model, with and without propagators for
$\mathrm{persistence}$ constraints, with that of notable optimal and satisficing domain-independent
planners. The former include, in no particular order, \emph{lmcut} \cite{helmert:lmcut:2009},
\emph{symbolic-bidirectional(sbd)} \cite{torralba:symba}, the baseline at the optimal track of  the
International Planning Competition (IPC) 2018,  \emph{cpddl} \cite{daniel:cpddl}, a
very efficient implementation of symbolic dynamic programming and many other pre-solving
techniques that analyse the structure of actions in the instance, \emph{delfi1}, a \emph{portfolio}
solver \cite{katz:delfi:2018} and winner of optimal planning track in IPC $2018$, and \emph{lisat} \cite{holler_behnke:lifted_sat_encoding}, a recently proposed lifted planner which has state-of-the-art performance on \emph{hard-to-ground}~(HTG) benchmark, it solves an encoding of lifted classical planning in propositional logic using a highly efficient SAT solver $\textsc{Kissat}$~\cite{fleury:kissat} written in $C$. Satisficing
planners include the \emph{satisficing} variant of \emph{lisat}, \emph{Madagascar} \cite{rintanen:madagascar:2012}, a SAT planner which was
the runner-up of Agile track in IPC $2014$, and \emph{lifted} implementations of BFWS planners,
BFWS([$\mathrm{R_X}, h^\mathrm{add}$]) and BFWS([$\mathrm{R_X}, h^\mathrm{ff}$])
\cite{correa:lifted} which have state-of-the-art \emph{satisficing} performance on HTG
benchmark~\cite{correa:lifted,lauer:htg:2021}. We use PL$^\text{add}_{R_x}$ and PL$^\text{ff}_{R_x}$ to denote the \emph{lifted} implementations of BFWS planners, and \textit{lisat} and $\overline{\textit{lisat}}$ to denote \emph{satisficing} implementation of \emph{lisat} with and without \emph{londex} \cite{chen_huang_xing_zhang:londex} constraints. All \emph{optimal} planners were configured to minimise the plan-length.



We evaluate all planners on the HTG benchmarks and the instances from the \emph{optimal track}
of the IPC~\cite{icaps:competitions}. Testing the planners on
the HTG instances is significant as the size of $U$ is very large, and as a result explicit grounding is either infeasible
or greatly stresses the implementation of key techniques (match trees, compilation into finite-domain representations)~\cite{helmert:fd}
that heuristic search planners rely on to be competitive.
Comparing the performance with IPC instances
allows us to control for implementation-dependant factors and also see how \textsc{CpSat} copes
with quickly growing numbers of variables and constraints. This is so because instances in the IPC benchmark tend to require
significantly higher number of plan steps for some domains (like \textit{logistics}).
We also tested a $3$-action lifted formulation of \emph{blocksworld} where
actions are \emph{move(x,y,z), move-to-table(x,y), and move-from-table(x,y)} which we think is
significant as it measures the sensitivity of planners to long-studied formulations of the same domain.

In addition to the IPC instances, we evaluate the planners on the multi-modal project scheduling problems (\emph{MMPSP}) from the 'j10' set in PSPLIB~\cite{mmpsp:psplib:sprecher1996}. The scheduling benchmarks are of particular interest to us since they exercise a different combinatorial structure than the IPC instances. While the IPC instances tend to have smaller and non-numeric \emph{sorts}, the scheduling instances usually have numeric duration and resources. To test the sensitivity of the planners to the distinguishing features of scheduling problems we performed an ablation study by scaling up the duration of the jobs in \emph{MMPSP} by a factor of $2$, $8$, and $16$, respectively.



\textbf{\textsc{CpSat} Hyperparameters}.
\textsc{CpSat} offers great flexibility to configure what pre-solving techniques, restarting policies,
and branching heuristics are used. In our experiments, we used the default branching
heuristic settings, and chose the \emph{luby} policy for managing restarts. \textsc{CpSat} default
branching heuristic settings tries first to fix values of literals appearing in CNF clauses
(which may have been lazily generated by some propagator) selecting the former with classical activity-based
variable selection heuristics~\cite{software:ortools}. Integer variables are only considered after all Boolean
literals are fixed. We use the \emph{linear scan} algorithm of Perron et
al.~\cite{software:ortools} to optimise Eq.~\eqref{pbo_objective}.

\textbf{Implementations of Lifted Causal Encodings}.
We have tested two implementations of constraints~\eqref{persist:3}. The first
uses the formulation in Eqs.~\ref{persists_formula_implementation}. The second one
uses the \emph{persistence propagator}. We will refer to them as \CP{0} and
\CP{\text{PP}}, respectively. We set $k_1$ to goal count for the CSP $T_{P,k_1}$ in our experiments, then used a \textit{luby} sequence to set $k_i$ for $i>1$. We use the same strategy for \textit{satisficing} planning except we scale the \emph{luby} sequence by a factor of $5$ and allocate a time-budget $B_i$ for the CSP $T_{P,k_i}$. $B_i$ is set such that the planner has sufficient time to explore a \emph{sliding window} over plan-lengths, $B_i := \min\{r,~r \cdot (k_i - {lb}_{i-1})/W\}$, where $r$ is the total remaining time-budget, ${lb}_{i-1}$ is the lower bound on the plan-length returned by the CSP $T_{P,k_{i-1}}$, and $W$ is the size of the \emph{sliding window} which is a planner parameter. We set $W=50$ in our experiments.

In order to assess the effectiveness of the functional transformation, we generated the functional representation using the method described at the end of the previous section, We refer to the encoding using functional representation as CP$_\text{0}^\textit{fn}$ and CP$_\text{PP}^\textit{fn}$.


\begin{table*}[ht]
	\centering
	{\tiny
	\scalebox{0.81}{
		\begin{tabular}
			{%
				@{}l%
				@{\extracolsep{5pt}}r%
				@{\extracolsep{5pt}}r%
        @{\extracolsep{5pt}}r%
        @{\extracolsep{5pt}}r%
        @{\extracolsep{5pt}}r%
        @{\extracolsep{5pt}}r%
        @{\extracolsep{5pt}}r%
        @{\extracolsep{5pt}}r%
				@{\extracolsep{5pt}}r|%
				@{\extracolsep{5pt}}r%
				@{\extracolsep{5pt}}r%
				@{\extracolsep{5pt}}r%
				@{\extracolsep{5pt}}r%
				@{\extracolsep{5pt}}r%
				@{\extracolsep{5pt}}r%
				@{\extracolsep{5pt}}r%
				@{}}
      \cmidrule{2-17}
			\textit{Hard-to-ground} &  \multicolumn{9}{c|}{Optimal Solution} & \multicolumn{6}{c}{Satisficing Solution} \\
			\midrule
      Domain{} &  CP$_{0}$ & CP$_{0}^\textit{fn}$ &  CP$_\text{PP}$ & CP$_\text{PP}^{\textit{fn}}$ & $\textit{lisat}$ & \textit{lmcut} & \textit{sbd} &\textit{cpddl} & \textit{delfi1} & CP$_{0}^\textit{fn}$ & CP$_\text{PP}^\textit{fn}$ & $\textit{lisat}$  & $\overline{\textit{lisat}}$  & MpC & PL$^\text{add}_{R_x}$ & PL$^\text{ff}_{R_x}$ \\
			\midrule
			blocks-3ops(40) & 40 & 40 & 40 & 40 & 40 & 0 & 0 & 0 & 0 & 40 & 40 & 40 & 12 & 0 & 10 & 10 \\
			blocks-4ops(40) & 40 & 40 & 40 & 40 & 40 & 10 & 0 & 1 & 0 & 40 & 40 & 40 & 20 & 4 & 19 & 17 \\
			childsnack(144) & 48 & 48 & 48 & 48 & 49 & 7 & 73 & 81 & 58 & 144 & 144 & 144 & 144 & 66 & 94 & 96 \\
			ged(156) & 23 & 30 & 22 & 31 & 35 & 18 & 12 & 14 & 18 & 37 & 29 & 58 & 28 & 30 & 156 & 156 \\
			ged-split(156) & 22 & 24 & 22 & 24 & 26 & 18 & 22 & 30 & 22 & 40 & 38 & 46 & 28 & 150 & 154 & 153 \\
			logistics(40) & 27 & 35 & 22 & 36 & 28 & 7 & 12 & 0 & 8 & 40 & 40 & 40 & 0 & 0 & 40 & 40 \\
			org-syn-MIT(18) & 18 & 18 & 18 & 18 & 18 & 2 & 2 & 13 & 2 & 18 & 18 & 18 & 10 & 0 & 18 & 18 \\
			org-syn-alk(18) & 18 & 18 & 18 & 18 & 18 & 18 & 18 & 18 & 18 & 18 & 18 & 18 & 18 & 0 & 18 & 18 \\
			org-syn-orig(20) & 13 & 13 & 15 & 15 & 20 & 0 & 1 & 2 & 1 & 5 & 9 & 14 & 1 & 0 & 13 & 13 \\
			pipes-tkg(50) & 15 & 15 & 15 & 17 & 20 & 8 & 12 & 13 & 10 & 23 & 25 & 10 & 23 & 10 & 45 & 46 \\
			rovers(40) & 3 & 7 & 3 & 8 & 3 & 7 & 2 & 0 & 2 & 11 & 10 & 4 & 0 & 0 & 39 & 40 \\
			visitall3D(60) & 25 & 25 & 24 & 34 & 35 & 33 & 12 & 12 & 24 & 33 & 49 & 46 & 39 & 12 & 57 & 57 \\
			visitall4D(60) & 23 & 23 & 23 & 35 & 34 & 16 & 6 & 6 & 6 & 30 & 44 & 48 & 36 & 6 & 42 & 41 \\
			visitall5D(60) & 27 & 26 & 26 & 33 & 33 & 11 & 0 & 0 & 0 & 32 & 44 & 54 & 38 & 0 & 40 & 40 \\
			\midrule
			Total(902) & 342 & 362 & 336 & 397 & 399 & 155 & 172 & 190 & 169 & 511 & 548 & 580 & 397 & 278 & 745 & 745 \\
			\bottomrule
      \\
		\end{tabular}
	}
		\caption{Coverage of different planners on \textit{hard-to-ground} benchmark domains}\label{table:htg}
	}

\end{table*}

\begin{figure*}[ht]
	\centering
	\includegraphics[width=1\linewidth]{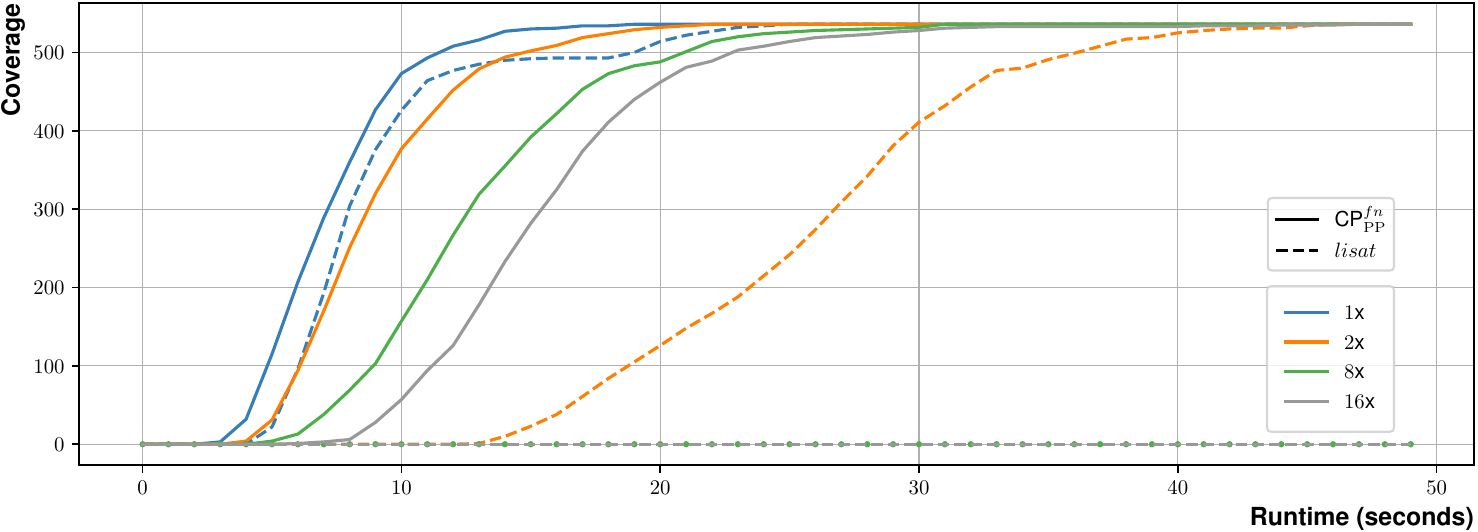}
	\caption{ Ablation study on the \emph{runtime} performance of $lisat$ and CP$^{fn}_\text{PP}$ by scaling up the duration of the jobs in \emph{MMPSP} by a factor of $2$, $8$ and $16$, respectively.}
	\label{fig:mmpsp_coverage}
\end{figure*}


\textbf{Performance over benchmarks}
We now discuss the observed performance (see
Table~\ref{table:htg})\footnote{Table~\ref{table:ipc} in the Appendix}
measured as \emph{coverage} or number of instances solved optimally (or sub-optimally) per problem domain
in each benchmark.
Also, we look closely at the sensitivity of $lisat$ and CP$_\text{PP}^{fn}$ to increasing duration of jobs in Figure~\ref{fig:mmpsp_coverage}.

The CP Planners perform very well on all formulations of the HTG instances of \emph{blocksworld} domain. Table~\ref{table:htg} reveals that every configuration of the CP planner solves the full set of HTG instances. Comparing the results on \emph{blocksworld} across to the IPC instance set\footref{table:ipc},
we see the CP$_\text{PP}^\textit{fn}$ planner performs strongly too, even when the IPC instances require much higher number of plan steps than the HTG ones.

The CP planners and baseline \textit{lisat} planners find feasible plans for many HTG \emph{childsnack} instances, significantly more than the state-of-the-art $\mathrm{PL}$
baselines, but have trouble finding optimal plans. For any given \emph{childsnack}
instance, there is a huge number of possible optimal plans that are permutations of each other. Without specific
guidance, our planners struggle with symmetries to obtain proofs of optimality quickly. Another structural feature
of optimal plans which is revealed to be problematic is the plan-length. Domains where plans require many actions
(HTG $ged$) are very challenging for our planners, with $lisat$ performing better and
the $\mathrm{PL}$
baseline having remarkably good performance in all of these instances, and
\emph{Madagascar} too when its techniques to bundle several actions apply.

In \emph{rovers}, the \emph{concise} encoding of dependency of preconditions on \emph{static relations} using the \emph{table constraints} (see Section~\ref{section:cp_implementation}) helps the CP$_\text{PP}^\textit{fn}$ achieve better optimizing performance than the baselines. The initial states of HTG \emph{rovers} instances may have $10,000$s of atoms to specify the \emph{static relations}, which otherwise would make the number of constraints~\eqref{init_assignemnts}
 and~\eqref{eq:persistence} to blow up.



The \emph{satisficing} performance of \emph{lisat} with \emph{londex} constraints against $\overline{\textit{lisat}}$ which does not use \emph{londex} shows \emph{impressive} gains in coverage by exploiting the structure of feasible plans to guide the search. The \emph{londex} implementation of \emph{lisat} restricts the \emph{supporter-supportee} distance to $1$, initially. If UNSAT, it increases the distance limit by $1$ and solves again. It repeats the procedure until \emph{timeout} or it finds a solution. This indicates a potential for improving the \emph{satisficing} performance of the CP encoding by designing and integrating planning specific heuristics into the CP solvers.

Overall, all \emph{optimizing} configurations of the CP planners perform much better than the baseline heuristic search planners on the HTG benchmark. The \emph{functional} transformation of the PDDL tasks shows \emph{impressive} performance gains, thereby highlighting the importance of a \emph{concise} encoding of the planning problems.
The performance of CP encoding with \emph{simple} transformation of
the PDDL task lags slightly behind \emph{lisat}. With the functional
transformation CP$_\text{PP}^\textit{fn}$ planner catches up to $lisat$ and their coverage is about the same.

In the \emph{MMPSP} instances, however, the CP approach shows its advantage over all other baseline planners.
As we can see in Figure~\ref{fig:mmpsp_coverage}, while the CP$_\text{PP}^\textit{fn}$ is slightly ahead of
\emph{lisat} for the original problems, scaling the durations only slowly
degrades CP$_\text{PP}^\textit{fn}$, but immediately makes a significant
difference to \emph{lisat} since it must encode much larger time domains,
while \textsc{CpSat} only lazily grounds the integers encoding times. We note that all the baseline heuristic planners exceeded the memory limit on the original problem set itself.


Lastly, in the IPC instances\footref{table:ipc}, the CP$_\text{PP}^\textit{fn}$ planner performs much better than the baseline \emph{lisat}. However, all CP planners as well as \emph{lisat} do not perform as well as the baseline heuristic search planners. With an exception of \emph{blocks-3ops} and \emph{organic-synthesis}, the coverage of the CP planners is lower than that of heuristic search planners.
This is an artifact of heuristic search being the dominant approach to planning~\cite{geffner2013concise}. Heuristic search planners work well with features showcased in the IPC instances, including significantly longer plans. On the other hand, their performance suffers in problem domains with large numeric types, specially over \emph{Resource-constrained planning(RCP)} problems~\cite{nakhost2012resource}.

\section{Related Work}
\label{related_work}

While causal encodings are the road less travelled in planning as satisfiability, there is significant related
work worth mentioning. Formulations based on the \emph{event calculus}~\cite{shanahan:sat} exist yet are rarely
acknowledged. We note that the event calculus $Initiates$ predicate corresponds
to our notion of output pins, $Holds$ corresponds to our notion of input pin, and $Terminates$ is very much
equivalent to our $\mathrm{persistence}$ predicate in Eq.~\eqref{persistence}. Constraint Programming was a natural
target for research seeking more compact encodings very early~\cite{vanbeek:cplan}, yet our most direct inspiration
was the \textsc{Cpt} planning system~\cite{vidal:06:cpt}, which in some of its
later versions\footnote{Personal communication with Hector Geffner.} used a notion of propagator  for its
\emph{causal link constraints}, very close to that later formalized by Ohrimenko et al. Previous work
have proposed grounded encodings that used KMS notion of \emph{operator splitting}~\cite{ernst:medic,robinson:compact},
which are structurally similar to our constraints for representing ground actions. Our formulation is entirely lifted,
and the only ground atoms we are forced to use are those present in initial and goal state formulas. We end
 acknowledging that the power of $\mathrm{element}(\cdot)$ constraints and their applications to
automated planning were revealed to us after the careful study of Francis et al.~\cite{francis:mda} and
Franc\`{e}s and Geffner~\cite{frances:fstrips}.

\section{Discussion}
\label{discussion}
This paper demonstrates that KMS notion of lifted causal encodings are an approach to planning
as satisfiability that has become viable thanks to the notable advances in CP over the last decade.
We have clearly barely scratched the surface of what is possible, as more propagator procedures
follow directly from the formulation in this paper, seeking powerful synergistic relations between
``planning-specific'' ones and general propagation algorithms. The clear limit to this approach lies in the number
of variables that need to be created. We also have not really looked into the possibility of integrating
or reformulating existing work that is known to enhance notably the scalability of planning as
satisfiability~\cite{rintanen:madagascar:2012}.

Alternatively, and perhaps ultimately too, we need to consider PDR-like
formulations~\cite{suda:pdr,eriksson:certification}, where we only need to construct a CP model
covering two plan steps. We observe that Suda's expedient of reimplementing the obligation
propagation mechanism as a ``Graphplan-like'' algorithm strongly suggests that a CP formulation
with suitably defined propagators could be very performant across PPIs with very diverse structure.
Also, PDR formulations seem to be key for certifying unsolvability~\cite{eriksson:certification}.



\bibliography{crossref,slots_enc}

\begin{thebibliography}{10}

\bibitem{abdulaziz:upper_bounds}
Mohammad Abdulaziz, Charles Gretton, and Michael Norrish.
\newblock {A State-Space Acyclicity Property for Exponentially Tighter Plan
  Length Bounds}.
\newblock In {\em Int'l Conference on Automated Planning and Scheduling
  (ICAPS)}, volume~27, pages 2--10, 2017.
\newblock \href {https://doi.org/10.1609/icaps.v27i1.13837}
  {\path{doi:10.1609/icaps.v27i1.13837}}.

\bibitem{balcazar:96}
Jos\'{e}~L. Balc\'{a}zar.
\newblock The complexity of searching implicit graphs.
\newblock {\em Artificial Intelligence}, 86(1):171--188, 1996.
\newblock \href {https://doi.org/10.1016/0004-3702(96)00014-8}
  {\path{doi:10.1016/0004-3702(96)00014-8}}.

\bibitem{blum:graphplan}
A.~Blum and M.~L. Furst.
\newblock Fast planning through planning graph analysis.
\newblock In {\em Int'l Conference on Automated Planning and Scheduling
  (ICAPS)}, 1995.

\bibitem{bonet:heuristics}
Blai Bonet and Hector Geffner.
\newblock Planning as heuristic search.
\newblock {\em Artificial Intelligence}, 129:5--33, 2001.

\bibitem{bourbaki:set_theory}
Nicholas Bourbaki.
\newblock {\em Elements of Mathematics: Theory of Sets}.
\newblock Springer-Verlag, 1968.

\bibitem{chen_huang_xing_zhang:londex}
Yixin Chen, Ruoyun Huang, Zhao Xing, and Weixiong Zhang.
\newblock Long-distance mutual exclusion for planning.
\newblock {\em Artificial Intelligence}, 173(2):365–391, 2009.

\bibitem{correa:lifted}
Augusto~B Corr{\^e}a, Florian Pommerening, Malte Helmert, and Guillem Frances.
\newblock Lifted successor generation using query optimization techniques.
\newblock In {\em Int'l Conference on Automated Planning and Scheduling
  (ICAPS)}, 2020.

\bibitem{eriksson:certification}
Salom\'{e} Eriksson and Malte Helmert.
\newblock {Certified Unsolvability for {S}{A}{T} Planning with Property
  Directed Reachability}.
\newblock In {\em Int'l Conference on Automated Planning and Scheduling
  (ICAPS)}, 2020.

\bibitem{ernst:medic}
Michael~D. Ernst, Todd~D. Millstein, and Daniel~S. Weld.
\newblock {Automatic {S}{A}{T}-Compilation of Planning Problems}.
\newblock In {\em Proc. Int'l Joint Conference on Artificial Intelligence
  (IJCAI)}, 1997.

\bibitem{daniel:cpddl}
Daniel Fi\v{s}er.
\newblock cpddl: a library for automated planning.
\newblock https://gitlab.com/danfis, 2022.

\bibitem{fleury:kissat}
ABKFM Fleury and Maximilian Heisinger.
\newblock Cadical, kissat, paracooba, plingeling and treengeling entering the
  sat competition 2020.
\newblock {\em SAT COMPETITION}, 2020:50, 2020.

\bibitem{frances:thesis}
Guillem Franc\`{e}s.
\newblock {\em Effective Planning with Expressive Languages}.
\newblock PhD thesis, DTIC, 2017.

\bibitem{frances:fstrips}
Guillem Frances and Hector Geffner.
\newblock Modeling and computatio in planning: Better heuristics from more
  expressive languages.
\newblock In {\em Int'l Conference on Automated Planning and Scheduling
  (ICAPS)}, 2015.

\bibitem{francis:mda}
Kathryn Francis, Jorge Navas, and Peter~J. Stuckey.
\newblock {Modelling Destructive Assignments}.
\newblock In {\em Int'l Conference on Principles and Practice of Constraint
  Programming (CP)}, pages 315--330, 2013.
\newblock \href {https://doi.org/10.1007/978-3-642-40627-0\_26}
  {\path{doi:10.1007/978-3-642-40627-0\_26}}.

\bibitem{geffner:planning_graphs}
Hector Geffner.
\newblock {Planning Graphs and Knowledge Compilation}.
\newblock In {\em Proc. of Principles of Knowledge Representation and Reasoning
  (KR)}, 2004.

\bibitem{geffner2013concise}
Hector Geffner and Blai Bonet.
\newblock A concise introduction to models and methods for automated planning.
\newblock {\em Synthesis Lectures on Artificial Intelligence and Machine
  Learning}, 8(1):1--141, 2013.

\bibitem{geffner2000admissible}
P~Haslum~H Geffner and Patrik Haslum.
\newblock Admissible heuristics for optimal planning.
\newblock In {\em Proceedings of the 5th Internat. Conf. of AI Planning Systems
  (AIPS 2000)}, pages 140--149, 2000.

\bibitem{green:gps}
C.~Green.
\newblock Application of theorem proving to problem solving.
\newblock In {\em Proc. Int'l Joint Conference on Artificial Intelligence
  (IJCAI)}, 1969.

\bibitem{haas:frame}
A.~Haas.
\newblock The case for domain-specific frame axioms.
\newblock In {\em The Frame Problem in Artificial Intelligence}. Morgan
  Kaufmann, 1987.

\bibitem{haslum2019introduction}
Patrik Haslum, Nir Lipovetzky, Daniele Magazzeni, and Christian Muise.
\newblock An introduction to the planning domain definition language.
\newblock {\em Synthesis Lectures on Artificial Intelligence and Machine
  Learning}, 13(2):1--187, 2019.

\bibitem{helmert:fd}
Malte Helmert.
\newblock The fast downward planning system.
\newblock {\em JAIR}, 26:191--246, 2006.

\bibitem{helmert:lmcut:2009}
Malte Helmert and Carmel Domshlak.
\newblock Landmarks, critical paths and abstractions: what's the difference
  anyway?
\newblock In {\em Int'l Conference on Automated Planning and Scheduling
  (ICAPS)}, 2009.

\bibitem{hoffmann:ff}
J{\"o}rg Hoffmann and Bernhard Nebel.
\newblock The ff planning system: Fast plan generation through heuristic
  search.
\newblock {\em Journal of Artificial Intelligence Research}, 14:253--302, 2001.

\bibitem{book:hooker:2012}
John~N Hooker et~al.
\newblock {\em Integrated methods for optimization}, volume 170.
\newblock Springer, 2012.

\bibitem{holler_behnke:lifted_sat_encoding}
Daniel Höller and Gregor Behnke.
\newblock Encoding lifted classical planning in propositional logic.
\newblock In {\em Proceedings of the International Conference on Automated
  Planning and Scheduling}, volume~32, page 134–144, 2022.

\bibitem{icaps:competitions}
International planning competitions -- classical tracks.
\newblock \url{https://www.icaps-conference.org/competitions/}, 1998-2018.
\newblock Accessed: 19-10-2022.

\bibitem{katz:delfi:2018}
Michael Katz, Shirin Sohrabi, Horst Samulowitz, and Silvan Sievers.
\newblock Delfi: Online planner selection for cost-optimal planning.
\newblock {\em IPC-9 planner abstracts}, pages 57--64, 2018.

\bibitem{kautz:encoding}
Henry Kautz, David McAllester, and Bart Selman.
\newblock {Encoding Plans in Propositional Logic}.
\newblock In {\em Proc. of Principles of Knowledge Representation and Reasoning
  (KR)}, 1996b.

\bibitem{kautz:sat_early}
Henry Kautz and Bart Selman.
\newblock Planning as satisfiability.
\newblock In {\em Proc. of European Conference in Artificial Intelligence
  (ECAI)}, 1992.

\bibitem{lauer:htg:2021}
Pascal Lauer, Alvaro Torralba, Daniel Fi{\v{s}}er, Daniel H{\"o}ller, Julia
  Wichlacz, and J{\"o}rg Hoffmann.
\newblock Polynomial-time in pddl input size: Making the delete relaxation
  feasible for lifted planning.
\newblock In {\em Int'l Conference on Automated Planning and Scheduling
  (ICAPS)}, 2021.

\bibitem{mcallester:nonlin}
David McAllester and David Rosenblitt.
\newblock Systematic nonlinear planning.
\newblock In {\em Proc. of the AAAI Conference (AAAI)}, 1991.

\bibitem{mccarthy:frame}
J.~McCarthy and P.~J. Hayes.
\newblock Some philosophical problems from the standpoint of {Artificial
  Intelligence}.
\newblock {\em Machine Intelligence}, 4:463--502, 1969.

\bibitem{nakhost2012resource}
Hootan Nakhost, J{\"o}rg Hoffmann, and Martin M{\"u}ller.
\newblock Resource-constrained planning: A monte carlo random walk approach.
\newblock In {\em Proceedings of the International Conference on Automated
  Planning and Scheduling}, volume~22, pages 181--189, 2012.

\bibitem{nieuwenhuis:diff_logic}
Robert Nieuwenhuis and Albert Oliveras.
\newblock {DPPL(T) with Exhaustive Theory Propagation and Its Applications to
  Difference Logic}.
\newblock In {\em Proc. of the Int'l Conf. on Computer Aided Verification
  (CAV)}, pages 321--334, 2005.
\newblock \href {https://doi.org/10.1007/11513988\_33}
  {\path{doi:10.1007/11513988\_33}}.

\bibitem{ohrimenko:lcg}
Olga Ohrimenko, Peter~J. Stuckey, and Michael Codish.
\newblock {Propagation via lazy clause generation}.
\newblock {\em Constraints}, 14(3):357--391, 2009.
\newblock \href {https://doi.org/10.1007/s10601-008-9064-x}
  {\path{doi:10.1007/s10601-008-9064-x}}.

\bibitem{pearl:heuristics}
Judea Pearl.
\newblock {\em Heuristics: Intelligent Search Strategies for Computer Problem
  Solving}.
\newblock Addison-Wellesley, 1984.

\bibitem{software:ortools}
Laurent Perron and Vincent Furnon.
\newblock Or-tools.
\newblock https://developers.google.com/optimization/, 2022.

\bibitem{richter:lama}
S.~Richter and M.~Westphal.
\newblock {The LAMA Planner: Guiding Cost-Based Anytime Planning with
  Landmarks}.
\newblock {\em Journal of Artificial Intelligence Research}, 39:127--177, 2010.
\newblock \href {https://doi.org/10.1613/jair.2972}
  {\path{doi:10.1613/jair.2972}}.

\bibitem{rintanen:regression}
Jussi Rintanen.
\newblock Regression for classical and nondeterministic planning.
\newblock {\em ECAI 2008}, page 568–572, 2008.
\newblock \href {https://doi.org/10.3233/978-1-58603-891-5-568}
  {\path{doi:10.3233/978-1-58603-891-5-568}}.

\bibitem{rintanen:admissible}
Jussi Rintanen.
\newblock {Planning with SAT, Admissible Heuristics and A*}.
\newblock In {\em Proc. Int'l Joint Conference on Artificial Intelligence
  (IJCAI)}, 2011.

\bibitem{rintanen:heuristics}
Jussi Rintanen.
\newblock {Planning as satisfiability: Heuristics}.
\newblock {\em Artificial Intelligence}, 193:45--86, 2012.
\newblock \href {https://doi.org/10.1016/j.artint.2012.08.001}
  {\path{doi:10.1016/j.artint.2012.08.001}}.

\bibitem{rintanen:madagascar:2012}
Jussi Rintanen.
\newblock Madagascar: Scalable planning with sat.
\newblock {\em Proceedings of the 8th International Planning Competition
  (IPC-2014)}, 21:1--5, 2014.

\bibitem{robinson:resolution_lifting}
J.~A. Robinson.
\newblock {A Machine-Oriented Logic Based on the Resolution Principle}.
\newblock {\em Journal of the ACM (JACM)}, 12(1):23--41, 1965.
\newblock \href {https://doi.org/10.1145/321250.321253}
  {\path{doi:10.1145/321250.321253}}.

\bibitem{robinson:compact}
Nathan Robinson, Charles Gretton, Duc-Nghia Pham, and Abdul Sattar.
\newblock {A Compact and Efficient {S}{A}{T} Encoding for Planning}.
\newblock In {\em Int'l Conference on Automated Planning and Scheduling
  (ICAPS)}, 2008.

\bibitem{seipp:lab}
Jendrik Seipp, Florian Pommerening, Silvan Sievers, and Malte Helmert.
\newblock {Downward} {Lab}.
\newblock https://doi.org/10.5281/zenodo.790461, 2017.

\bibitem{shanahan:sat}
Murray Shanahan and Mark Witkowski.
\newblock {Event Calculus Planning Through Satisfiability}.
\newblock {\em Journal of Logic and Computation}, 14(5):731--745, 2004.
\newblock \href {https://doi.org/10.1093/logcom/14.5.731}
  {\path{doi:10.1093/logcom/14.5.731}}.

\bibitem{mmpsp:psplib:sprecher1996}
A~Sprecher and R~Kolisch.
\newblock Psplib—a project scheduling problem library.
\newblock {\em Eur. J. Oper. Res}, 96:205--216, 1996.

\bibitem{streeter:sat_search}
Matthew Streeter and Stephen~F. Smith.
\newblock {Using Decision Procedures Efficiently for Optimization}.
\newblock In {\em Int'l Conference on Automated Planning and Scheduling
  (ICAPS)}, 2007.

\bibitem{stuckey:challenge}
Peter~J. Stuckey, T.~Feydy, A.~Schutt, Guido Tack, and J.~Fischer.
\newblock The {MiniZinc} challenge 2008-2013.
\newblock {\em AI Magazine}, 35(2):55--60, 2014.

\bibitem{suda:pdr}
M~Suda.
\newblock {Property Directed Reachability for Automated Planning}.
\newblock {\em Journal of Artificial Intelligence Research}, 50:265--319, 2014.
\newblock \href {https://doi.org/10.1613/jair.4231}
  {\path{doi:10.1613/jair.4231}}.

\bibitem{torralba:symba}
Alvaro Torralba, Vidal Alc\'{a}zar, Peter Kissmann, and Stefan Edelkamp.
\newblock {Efficient symbolic search for cost-optimal planning}.
\newblock {\em Artificial Intelligence}, 242:52--79, 2017.
\newblock \href {https://doi.org/10.1016/j.artint.2016.10.001}
  {\path{doi:10.1016/j.artint.2016.10.001}}.

\bibitem{vanbeek:cplan}
Peter {van Beek} and Xinguang Chen.
\newblock {CPlan: A Constraint Programming Approach to Planning}.
\newblock In {\em Proc. of the AAAI Conference (AAAI)}, 1999.

\bibitem{vidal:06:cpt}
Vincent Vidal and Hector Geffner.
\newblock Branching and pruning: an optimal temporal {P}{O}{C}{L} planner based
  on constraint programming.
\newblock {\em Artificial Intelligence}, 170:298--335, 2006.

\end{thebibliography}

\appendix

\section{\texorpdfstring{$h^m$}~ heuristic over first-order logic existential formulae} \label{h_m:fol_existential}

In this section, we present an account of $h^m$ \emph{admissible} heuristic \cite{geffner2000admissible} and its extension to the first-order logic existential formulas which we then use to test the reachability of a first-order logic formula. The heuristic function $h^m$ is an \emph{admissible approximation} of the
optimal heuristic function $h^*:{\cal{A}} \mapsto \mathbb{N}$, defined over
a conjunction of positive and negative \emph{ground} atoms $
{\cal{A}}:={{\cal P} \times t_1 \times \ldots \times t_{K_{\cal P}} \times
  \{\top, \bot\}}$. $h^*$ is defined using the \emph{regression} model of
the planning problem, in which we regress \emph{backwards} from a
\emph{goal} formula using a \emph{regression function} $rg_a: {\cal A}
\mapsto {\cal A}$ with respect to a ground action $a \in {Act \times
  U^{K_\alpha}}$  \cite{rintanen:regression}. We extend the
\emph{regression} function $rg_a$ to the first-order logic existential
formulas and use it to define $h^m$ for the first-order logic formulas.

Let $\psi$ be a first-order logic existential formula, $\psi := \exists \bar{x}, \psi^L \land \psi^{EQ}$, where $\bar{x}:=(x_1, x_2,\ldots, x_n)$ is a vector of parameters, $\psi^L$ is a conjunction of literals whose interpretation is set by the states $s \in S^R_P$, and $\psi^{EQ}$ is an equality-logic formula. We denote the set of literals in a formula $\phi$ by $\emph{lits}_{\phi}$, the predicate and the argument variables of a literal $l \in \emph{lits}_{\phi}$ by $\emph{predicate}_l$ and $\overline{arg}_l$, respectively, and the polarity of $l$ by $\emph{polarity}_l$.

The \emph{regression} of the formula $\psi$ with respect to an action schema $\alpha := \left<\pre{\alpha}, \eff{\alpha} \right>$ involves identifying the \emph{supporter-supportee} pairs and the \emph{inconsistent} literal pairs between $\eff{\alpha}$ and $\psi^L$. A \emph{supporter-supportee} relationship holds between $l \in \emph{lits}_{\eff{\alpha}}$ and $l' \in \emph{lits}_{\psi^L}$ iff the predicate and the arguments of $l$ \emph{match} that of $l'$ and they have the \emph{same} \emph{polarity}. On the other hand, a literal $l \in \emph{lits}_{\eff{\alpha}}$ is \emph{inconsistent} with $l' \in \emph{lits}_{\psi^L}$ iff the predicate and the arguments of $l$ \emph{match} that of $l'$ but they have the \emph{opposite} \emph{polarity}.

While \emph{regressing} with respect to action schema $\alpha$, we need to consider every possible combination of \emph{supporter-supportee} pairs, i.e. all subsets of the set of \emph{potential} supporter-supportee pairs $\emph{SP}:=\{(l, l') \mid \emph{polarity}_l=\emph{polarity}_{l'},~\emph{predicate}_{l}=\emph{predicate}_{l'},\ l \in \emph{lits}_{\eff{\alpha}}, l' \in \emph{lits}_{\psi^L}\}$, then for each literal pair $(l,l')$ in the subset we add the constraint $\bigwedge_{i=1}^{K_{\cal P}} {arg}_{li} = {arg}_{l'i}$ to the formula obtained through regression. Thus, the \emph{regression} of $\psi$ with respect to $\alpha$ produces a set of formulas, one for each subset in $\emph{SP}$. Similarly, for each pair in the set of \emph{potential} inconsistent pairs $\textit{IP}:=\{(l, l') \mid \emph{polarity}_l \neq \emph{polarity}_{l'},~\emph{predicate}_{l}=\emph{predicate}_{l'},\ l \in \emph{lits}_{\eff{\alpha}}, l' \in \emph{lits}_{\psi^L}\}$, we add the constraint $\bigvee_{i=1}^{K_{\cal P}} {arg}_{li} \neq {arg}_{l'i}$ to the regression formula of $\psi$. We now present the definition of the \emph{regression function} over first-order logic formulas.

\begin{subequations}
	\label{regression}
	\begin{align}
		&rg_\alpha(\psi)  :=\{\exists \bar{x},\ \widehat{rg}_\alpha(\psi^L, \widetilde{\emph{SP}}) \land \widecheck{rg}_\alpha(\psi^{EQ}, \widetilde{\emph{SP}}) \mid \widetilde{\emph{SP}} \subseteq \emph{SP}\}\\
		&\widehat{rg}_\alpha(\psi^L, \widetilde{\emph{SP}})  := \text{Pre}_\alpha \land \bigwedge_{l \in \{\emph{lits}_{\psi^L} \backslash \{l' \mid (l,l') \in \widetilde{\emph{SP}}\}\}} l \\
		&\widecheck{rg}_\alpha( \psi^{EQ}, \widetilde{\emph{SP}})  :=  \psi'^{EQ} \land \bigwedge_{(l,l') \in \widetilde{\emph{SP}}} \land_{i=1}^{K_{\cal P}} arg_{li} = arg_{l'i} \land \bigwedge_{(l,l') \in \textit{IP}} \lor_{i=1}^{K_{\cal P}} arg_{li} \neq arg_{l'i}
	\end{align}
\end{subequations}
\noindent
where we obtain the regression of $\psi^L$ by removing the \emph{supportees} in the set $\widetilde{\emph{SP}}$ from $\psi^L$ and adding the precondition of $\alpha$. The regression of $\psi^{EQ}$ involves a reduction into a \emph{canonical} form $\psi'^{EQ}$  by setting the equality atoms whose arguments do not appear in the arguments of the regression  $\widehat{rg}_\alpha(\psi^L, \widetilde{\emph{SP}})$ to \emph{true}. Then, we add two equality logic formulas, first of which \emph{binds} the arguments of \emph{supporter-supportee} pairs in $\widetilde{\emph{SP}}$ and the second disallows \emph{inconsistent} assignments to the arguments of literals pairs in \emph{IP}.

The $h^m$ heuristics for $\psi := \exists \bar{x}, \psi^L \land \psi^{EQ}$ is defined using regression as follows

\begin{align}
    h^m(\psi):=
\begin{cases}
    0, & s_0 \models \psi\\
    min_{\psi' \in rg_\alpha(\psi), \alpha \in \textit{Act}} h^m(\psi') + cost(\alpha), & |\textit{lits}(\psi^L)| \leq m\\
    max_{\psi \vdash \psi', |\textit{lits}(\psi'^{\textit{L}})|\leq m} h^m(\psi'),& \textit{otherwise}
\end{cases}
\end{align}

$h^m(\psi_\mathsf{P})$ can be efficiently computed using dynamic programming with \textit{memoization}, and $s_0 \models \psi$ can be encoded as a \text{CP} program with \textit{equality and table constraints}. For a fixed value of $m$, the complexity of the above procedure is low polynomial in the number of nodes, i.e. the number of first-order logic formulas $O({|{\cal P}|^m}\cdot 2^{K_{\cal P}})$. An important property of the above procedure is that it is \emph{entirely lifted}, i.e. no \emph{ground} atom would occur in the formulas obtained through regression if no action schema has \emph{ground} atoms. This is specially useful in \emph{hard-to-ground}(HTG) domains where the size of \emph{ground} theory renders the \emph{translation} methods \cite{helmert:fd} used by most planners intractable.

\section{Proofs}

\begin{theorem}[Systematicity~\cite{mcallester:nonlin}]
	\label{theorem:systematicity}
	Let $T_{P,N_{\pi}}$ $=$ $({\cal X}$, ${\cal C})$ be the CSP given by the decision variables ${\cal X}$ in Table~1
	and set of constraints ${\cal C}$~$(4)$--$(12)$, for some suitable choice of $N_{\pi}$.
	There exists an assignment $\xi$ onto variables ${\cal X}$ that satisfies every constraint in ${\cal C}$ if an only if there exists
	a feasible solution $\sigma$ to the PPI $P$, whose actions are given by slots, argument, input and output pin variables values.
\end{theorem}

\begin{proof}
	It follows trivially from the definitions given in the previous sections that assignments $\xi$ encode finite trajectories
	$\sigma=s_0 a_1 s_1 \ldots s_{N_\pi}$.
	To prove the forward direction, it suffices to observe that (1) constraints~$(5)$ on input
	and output pins for a slot $i$ define \emph{implicitly} sets of pairs of states $(s_{i-1}, s_i) \in \to_{act_i}$, the set of transitions
	corresponding to the schema $act_i = \alpha$ assigned to the slot, (2) constraints~$(11)$
	ensure that for $i=1$ the predecessor of $s_1$ corresponds with the initial state $s_0$ given in the definition of the PPI $P$, and (3)
	the last state in the trajectory given by $\xi$ the equality atoms in $Goal$. To prove the reverse direction, we note that each pair
	of consecutive states $s_{i-1}$ and $s_i$ must belong to exactly one of the transition sets $\to_{\alpha}$.
	From the definition of this set the schema, argument and pins assignments follow directly.$\square$
	\end{proof}

\begin{theorem}
	\label{theorem:function_transformation}
	If ${\cal R}_P$ is a \textit{mapping} in all possible interpretations $s \in S^R_\Pi$, then ${\Pi'}$, the transformation of the problem ${\Pi}$ which encodes the predicate $P$ as a function $f_P : \domain{x} \mapsto \domain{y}$, has the same set of reachable states as $\Pi$, i.e. $S^R_\Pi = S^R_{\Pi'}$
\end{theorem}

\begin{proof}
	If the \textit{right-unique} and \textit{left-total} properties \textit{hold} for ${\cal R}_P$ in all $s \in S^R_{\Pi}$, then applying the functional transformation would not alter the reachable state space since the functional transformation \emph{implicitly} enforces the same conditions, i.e. $\forall\ x_1, x_2 \in \domain{x}, y_1,y_2 \in \domain{y},\ (f_P(x_1)=y_1) \land (f_P(x_2)=y_2) \land (x_1=x_2) \rightarrow (y_1=y_2)$ and $\forall\ x \in \domain{x},\ \exists\ y \in \domain{y},\ (f_P(x)=y)$.$\square$
\end{proof}

\section{Additional Results: Figures and Tables}

Figure~\ref{fig:performance_profiles} depicts the performance profiles of CP$_0^\textit{fn} \mathrm{(1)}$, CP$_\text{PP}^\textit{fn}  \mathrm{(2)}$, \textit{cpddl}$\mathrm{(3)}$ and \textit{lmcut}$\mathrm{(4)}$, and
Table~\ref{table:ipc} shows of \emph{coverage} of baseline and CP planners on the IPC benchmarks.

\begin{figure*}[ht!]
	\centering
	\includegraphics[width=1\linewidth]{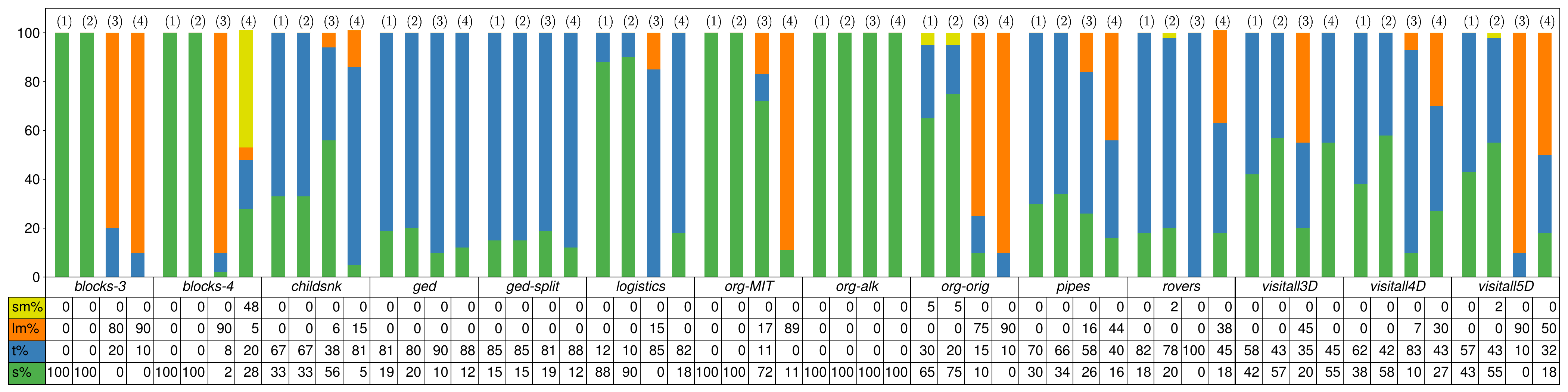}
	\caption{ Plot depicting performance profiles of CP$_0^\textit{fn} \mathrm{(1)}$, CP$_\text{PP}^\textit{fn}  \mathrm{(2)}$, \textit{cpddl}$\mathrm{(3)}$, and \textit{lmcut}$\mathrm{(4)}$ on the HTG benchmark set. The percentage of instances which \emph{solved $\mathrm{(s\%)}$ (optimally)}, reported \emph{memout} during \emph{loading} step $\mathrm{(lm\%)}$, reported \emph{memout} during the search $\mathrm{(sm\%)}$, and \emph{timed-out} $\mathrm{(t\%)}$ are shown. }
	\label{fig:performance_profiles}
\end{figure*}

\begin{table*}[ht!]
	\centering
	\def\arraystretch{0.94}
	{\footnotesize
	\scalebox{0.9}{
		\begin{tabular}
			{%
				@{}l%
				@{\extracolsep{5pt}}r%
				@{\extracolsep{5pt}}r%
        @{\extracolsep{5pt}}r%
        @{\extracolsep{5pt}}r%
        @{\extracolsep{5pt}}r%
        @{\extracolsep{5pt}}r%
        @{\extracolsep{5pt}}r%
        @{\extracolsep{5pt}}r%
				@{\extracolsep{5pt}}r
				%
				@{}}

      \cmidrule{2-10}
			\textit{IPCs(opt)} &  \multicolumn{9}{c}{Optimal Solution}
			\\
			\midrule
      Domain{} & $\#$Instances & CP$_{0}$ & CP$_\text{PP}$ & CP$_\text{PP}^\textit{fn}$ & $\textsc{LiSAT}$ & \textit{lmcut} & \textit{sbd} &\textit{cpddl} & \textit{delfi1} \\
			\cmidrule{2-10}

			\multicolumn{10}{c}{\textit{Not-grounded} - action schemas do not have ground atoms} \\
			\midrule
			barman-opt11                  &     20 &                                           0 &                                            0 &                                                 0 &                            0 &                 4 &               9 &                     12 &                  7 \\
			barman-opt14                  &     14 &                                           0 &                                            0 &                                                 0 &                            0 &                 0 &               3 &                      6 &                  2 \\
			blocks                               &     35 &                                          17 &                                           19 &                                                28 &                           26 &                28 &              30 &                     31 &                 27 \\
			blocks-3ops                          &     35 &                                          19 &                                           23 &                                                35 &                            5 &                26 &              25 &                     30 &                 20 \\
			childsnack-opt14              &     20 &                                           0 &                                            0 &                                                 0 &                            0 &                 0 &               4 &                      4 &                  6 \\
			data-network-opt18            &     20 &                                          14 &                                           15 &                                                15 &                            0*&                20 &              17 &                     17 &                 17 \\
			depot                                &     22 &                                           1 &                                            1 &                                                 2 &                            2 &                 7 &               5 &                      7 &                 12 \\
			driverlog                            &     20 &                                           4 &                                            4 &                                                 4 &                            5 &                14 &              12 &                     12 &                 15 \\
			elevators-opt08               &     30 &                                           2 &                                            3 &                                                 5 &                            6 &                20 &              22 &                     23 &                 20 \\
			elevators-opt11               &     20 &                                           1 &                                            1 &                                                 3 &                            4 &                17 &              18 &                     18 &                 16 \\
			floortile-opt11               &     20 &                                           0 &                                            0 &                                                 0 &                            0 &                 6 &              14 &                     14 &                 12 \\
			floortile-opt14               &     20 &                                           0 &                                            0 &                                                 0 &                            0 &                 5 &              20 &                     20 &                 17 \\
			freecell                             &     80 &                                           6 &                                            6 &                                                 6 &                           12 &                15 &              21 &                     22 &                 18 \\
			ged-opt14                     &     20 &                                          19 &                                           19 &                                                20 &                           20 &                20 &              20 &                     20 &                 20 \\
			grid                                 &      5 &                                           1 &                                            1 &                                                 2 &                            1 &                 2 &               2 &                      2 &                  3 \\
			gripper                              &     20 &                                           1 &                                            2 &                                                 2 &                            2 &                 7 &              20 &                     20 &                 20 \\
			hiking-opt14                  &     20 &                                           4 &                                            4 &                                                 4 &                            8 &                 9 &              14 &                     15 &                 19 \\
			logistics00                          &     28 &                                           2 &                                            2 &                                                 3 &                            6 &                20 &              18 &                     19 &                 21 \\
			logistics98                          &     35 &                                           0 &                                            0 &                                                 1 &                            1 &                 6 &               5 &                      5 &                  8 \\
			miconic                              &    150 &                                          22 &                                           23 &                                                30 &                           32 &               141 &             104 &                    104 &                136 \\
			mprime                               &     35 &                                          31 &                                           32 &                                                33 &                           33 &                22 &              23 &                     25 &                 25 \\
			mystery                              &     30 &                                          18 &                                           18 &                                                18 &                           19 &                17 &              13 &                     15 &                 17 \\
			nomystery-opt11               &     20 &                                           6 &                                            6 &                                                 6 &                           10 &                14 &              13 &                     14 &                 14 \\
			organic-synthesis-opt18       &     20 &                                          20 &                                           20 &                                                20 &                           20 &                 7 &               7 &                     13 &                  8 \\
			organic-synthesis-split-opt18 &     20 &                                           8 &                                           12 &                                                12 &                           10 &                16 &              13 &                      8 &                 12 \\
			parking-opt11                 &     20 &                                           1 &                                            1 &                                                 1 &                            1 &                 2 &               1 &                      1 &                  5 \\
			parking-opt14                 &     20 &                                           0 &                                            0 &                                                 0 &                            0 &                 3 &               0 &                      1 &                  7 \\
			pegsol-08                     &     30 &                                           9 &                                            8 &                                                11 &                           20 &                27 &              28 &                     29 &                 28 \\
			pegsol-opt11                  &     20 &                                           1 &                                            1 &                                                 1 &                            6 &                17 &              18 &                     19 &                 18 \\
			pipesworld-notankage                 &     50 &                                          12 &                                           12 &                                                12 &                           14 &                17 &              15 &                     15 &                 25 \\
			pipesworld-tankage                   &     50 &                                           9 &                                           10 &                                                11 &                           11 &                12 &              16 &                     17 &                 22 \\
			rovers                               &     40 &                                           4 &                                            4 &                                                 4 &                            4 &                 8 &              14 &                     14 &                 12 \\
			satellite                            &     36 &                                           3 &                                            3 &                                                 4 &                            5 &                 7 &              10 &                     11 &                 14 \\
			scanalyzer-08                 &     30 &                                          10 &                                            9 &                                                13 &                           12 &                 9 &              13 &                     13 &                 17 \\
			scanalyzer-opt11              &     20 &                                           7 &                                            6 &                                                10 &                            9 &                 6 &              10 &                     10 &                 13 \\
			sokoban-opt08                 &     30 &                                           0 &                                            0 &                                                 0 &                            2 &                24 &              25 &                     28 &                 28 \\
			sokoban-opt11                 &     20 &                                           0 &                                            0 &                                                 0 &                            0 &                19 &              20 &                     20 &                 20 \\
			spider-opt18                  &     20 &                                           0 &                                            0 &                                                 0 &                            0 &                 6 &               6 &                      6 &                  8 \\
			storage                              &     30 &                                          12 &                                           11 &                                                13 &                            0* &                15 &              14 &                     15 &                 17 \\
			termes-opt18                  &     20 &                                           0 &                                            0 &                                                 0 &                            0 &                 5 &              16 &                     16 &                 12 \\
			tetris-opt14                  &     17 &                                           2 &                                            2 &                                                 3 &                            3 &                 5 &              10 &                     12 &                 13 \\
			tidybot-opt11                 &     20 &                                           1 &                                            3 &                                                 3 &                            0* &                14 &              12 &                     11 &                 17 \\
			tidybot-opt14                 &     20 &                                           0 &                                            0 &                                                 0 &                            0* &                 9 &               5 &                      7 &                 13 \\
			tpp                                  &     30 &                                           4 &                                            4 &                                                 4 &                            5 &                 7 &               8 &                      8 &                 11 \\
			transport-opt08               &     30 &                                           6 &                                            6 &                                                 6 &                            6 &                12 &              14 &                     14 &                 13 \\
			transport-opt11               &     20 &                                           1 &                                            1 &                                                 1 &                            1 &                 8 &              10 &                     11 &                 10 \\
			transport-opt14               &     20 &                                           1 &                                            1 &                                                 1 &                            2 &                 7 &               9 &                     10 &                  9 \\
			visitall-opt11                &     20 &                                           8 &                                            9 &                                                11 &                           11 &                10 &              12 &                     12 &                 17 \\
			visitall-opt14                &     20 &                                           2 &                                            3 &                                                 5 &                            5 &                 5 &               6 &                      6 &                 13 \\
			woodworking-opt08             &     30 &                                           7 &                                            7 &                                                 7 &                           10 &                17 &              30 &                     29 &                 28 \\
			woodworking-opt11             &     20 &                                           2 &                                            2 &                                                 2 &                            5 &                12 &              20 &                     20 &                 20 \\
			zenotravel                           &     20 &                                           7 &                                            7 &                                                 8 &                            9 &                13 &               9 &                      0 &                 12 \\

			\midrule
			\multicolumn{10}{c}{\textit{Partially-grounded} - action schemas have a few ground atoms} \\
			\midrule
			agricola-opt18                &     20 &                                           0 &                                            0 &                                                 0 &                            3 &                 0 &              14 &                     12 &                 10 \\
			airport                              &     50 &                                           6 &                                            7 &                                                 7 &                            7 &                28 &              23 &                     24 &                 23 \\
			movie                                &     30 &                                          30 &                                           30 &                                                30 &                            0* &                30 &              30 &                     30 &                 30 \\
			openstacks-opt08              &     30 &                                           0 &                                            1 &                                                 1 &                            2 &                 8 &              30 &                     30 &                 30 \\
			openstacks-opt11              &     20 &                                           0 &                                            0 &                                                 0 &                            0 &                 3 &              20 &                     20 &                 20 \\
			openstacks-opt14              &     20 &                                           0 &                                            0 &                                                 0 &                            0 &                 0 &              15 &                     16 &                 12 \\
			parcprinter-08                &     30 &                                           6 &                                            6 &                                                 6 &                            0* &                22 &              30 &                     30 &                 30 \\
			parcprinter-opt11             &     20 &                                           3 &                                            3 &                                                 3 &                            0* &                16 &              20 &                     20 &                 20 \\
			pathways                             &     30 &                                           3 &                                            4 &                                                 4 &                            0* &                 5 &               5 &                      5 &                  5 \\
			snake-opt18                   &     20 &                                           3 &                                            3 &                                                 6 &                            0* &                 7 &               3 &                      5 &                 11 \\

			\midrule
			\multicolumn{10}{c}{\textit{Fully-grounded} - action schemas only have ground atoms} \\
			\midrule
			openstacks                    &     30 &                                           0 &                                            0 &                                                 0 &                            0 &                 7 &              18 &                     17 &                 11 \\
			petri-net-alignment-opt18     &     20 &                                           0 &                                            0 &                                                 0 &                            0 &                 3 &              18 &                     20 &                 20 \\
			psr-small                            &     50 &                                          44 &                                           46 &                                                46 &                            0* &                49 &              50 &                     50 &                 50 \\
			trucks                        &     30 &                                           2 &                                            2 &                                                 2 &                            0* &                10 &              11 &                     14 &                  9 \\
			\midrule
			Total                                &   1862 &                                         402 &                                          423 &                                               485 &                          375 &               927 &            1090 &                   1124 &               1195 \\
			\bottomrule
      \\
		\end{tabular}
	}

		\caption{Coverage of different planners on \textit{IPCs -- optimal track} benchmark domains. * indicates the domains in which the preprocessing (\emph{parsing and encoding}) step of \emph{lisat} rendered the instances \emph{unsolvable}.}\label{table:ipc}
	}

\end{table*}

\end{document}